\documentclass[9.5pt,journal,cspaper,compsoc]{IEEEtran}

\usepackage{times}
\usepackage{wrapfig}
\usepackage[T1]{fontenc}
\usepackage{color}
\usepackage{array}
\usepackage{verbatim}
\usepackage{float}
\usepackage{footnote}
\usepackage{multirow}
\usepackage{amsmath}
\usepackage{amssymb}
\usepackage{graphicx}
\usepackage{subfig}
\usepackage{caption}
\usepackage{algorithm}
\usepackage{algorithmic}
\usepackage[ruled,vlined,algo2e]{algorithm2e}
\usepackage[utf8]{inputenc}
\usepackage[english]{babel}
\usepackage{amsthm}
\usepackage{booktabs}
\usepackage{adjustbox}
\usepackage{enumitem}

\newtheorem{prop}{Proposition}

\DeclareMathOperator*{\argmin}{arg\,min}

\newcommand{\bfsection}[1]{\vspace*{0.1cm}\noindent\textbf{#1.}}

\usepackage{mathtools}

\def\resp{\emph{resp. }}
\def\eg{\emph{e.g. }}
\def\ie{\emph{i.e. }}

\def\wrt{\emph{w.r.t. }}
\def\etal{\emph{et al. }}

\usepackage{hyperref}
\hypersetup{
    colorlinks=true,
    linkcolor=blue,
    filecolor=magenta,      
    urlcolor=cyan,
}
 
\ifCLASSOPTIONcompsoc
  \usepackage[nocompress]{cite}
\else
  \usepackage{cite}
\fi

\ifCLASSINFOpdf

\else

\fi

\begin{document}
\title{Revisiting 2D Convolutional Neural Networks for Graph-based Applications}

\author{Yecheng~Lyu,~\IEEEmembership{Student Member,~IEEE,}
        Xinming~Huang,~\IEEEmembership{Senior Member,~IEEE,}
        and~Ziming~Zhang,~\IEEEmembership{Member,~IEEE}
\IEEEcompsocitemizethanks{\IEEEcompsocthanksitem Yecheng Lyu, Dr. Xinming Huang and Dr. Ziming Zhang are with the Department of Electrical 
and Computer Engineering, Worcester Polytechnic Institute, Worcester, MA 01609, USA. 
\newline Email: \{ylyu,xhuang,zzhang15\}@wpi.edu

\IEEEcompsocthanksitem Part of this work was done when Lyu and Zhang worked at Mitsubishi Electric Research Laboratories (MERL).
}
}

\markboth{IEEE Transaction on Pattern Analysis and Machine Intelligence}{Lyu \etal}

\IEEEtitleabstractindextext{%
\begin{abstract}
Graph convolutional networks (GCNs) are widely used in graph-based applications such as graph classification and segmentation. However, current GCNs have limitations on implementation such as network architectures due to their irregular inputs. In contrast, convolutional neural networks (CNNs) are capable of extracting rich features from large-scale input data, but they do not support general graph inputs. To bridge the gap between GCNs and CNNs, in this paper we study the problem of how to effectively and efficiently map general graphs to 2D grids that CNNs can be directly applied to, while preserving graph topology as much as possible. We therefore propose two novel graph-to-grid mapping schemes, namely, {\em graph-preserving grid layout (GPGL)} and its extension {\em Hierarchical GPGL (H-GPGL)} for computational efficiency. We formulate the GPGL problem as integer programming and further propose an approximate yet efficient solver based on a penalized Kamada-Kawai method, a well-known optimization algorithm in 2D graph drawing. We propose a novel vertex separation penalty that encourages graph vertices to lay on the grid without any overlap. Along with this image representation, even extra 2D maxpooling layers contribute to the PointNet, a widely applied point-based neural network. We demonstrate the empirical success of GPGL on general graph classification with small graphs and H-GPGL on 3D point cloud segmentation with large graphs, based on 2D CNNs including VGG16, ResNet50 and multi-scale maxout (MSM) CNN. 
\end{abstract}

}

\maketitle

\section{Introduction}\label{sec:introduction}

\IEEEPARstart{G}{raph} data processing using neural networks has been broadly attracting more and more research interests recently. Graph convolutional networks (GCNs) \cite{defferrard2016convolutional, kipf2016semi, hamilton2017inductive, bronstein2017geometric, chen2018fastgcn, gao2019graph, wu2019simplifying, wu2020comprehensive, morris2019weisfeiler},\cite{zhao2018work,jiang2019gaussian,xuan2019subgraph,niepert2016learning,al2019ddgk,gao2019hGANet,xu2018powerful,ivanov2018anonymous,chen2019dagcn,tixier2019graph,xinyi2018capsule,zhang2019quantum,jin2018discriminative,zhao2019learning,ma2019graph,knyazev2018spectral,jia2019graph,yanardag2015deep,neumann2016propagation},\cite{kriege2016valid,atamna2019spi,corcoran2020function,verma2018graph},\cite{togninalli2019wasserstein,li2019semi,kondor2016multiscale},\cite{lyu2021treernn,wang2020second,hu2020going, jaume2019edgnn} are a family of graph-based neural networks that extend convolutional neural networks (CNNs) to extract local features in general graphs with irregular input structures. The irregularity of a graph, including the orderless nodes and connections, however, makes the GCNs difficult to design as well as training from local patterns. In general, a GCN has two key operations to compute feature representations for the nodes in a graph, namely aggregation and transformation. That is, the feature representation of a node is computed as an aggregate of the feature representations of its neighbors before it is transformed by applying the weights and activation functions. To deal with graph irregularity the adjacency matrix is fed into the aggregation function to encode the topology. The weights are shared by all the nodes to perform convolutions to their neighborhood nodes.

CNNs, equipped with tens or hundreds of layers and millions of parameters thanks to their well designed operators upon 2D grid inputs, succeed in many research areas such as speech recognition \cite{abdel2014convolutional} and computer vision \cite{he2016deep}. In the grid inputs such as images, the highly ordered adjacency and unique layout bring the ability of designing accurate, efficient and salable convolutional kernels and pooling operations over large-scale data. How to apply CNNs to general graph inputs, however, still remains elusive.

\begin{figure}[t]
	\centering
	\includegraphics[width=\columnwidth]{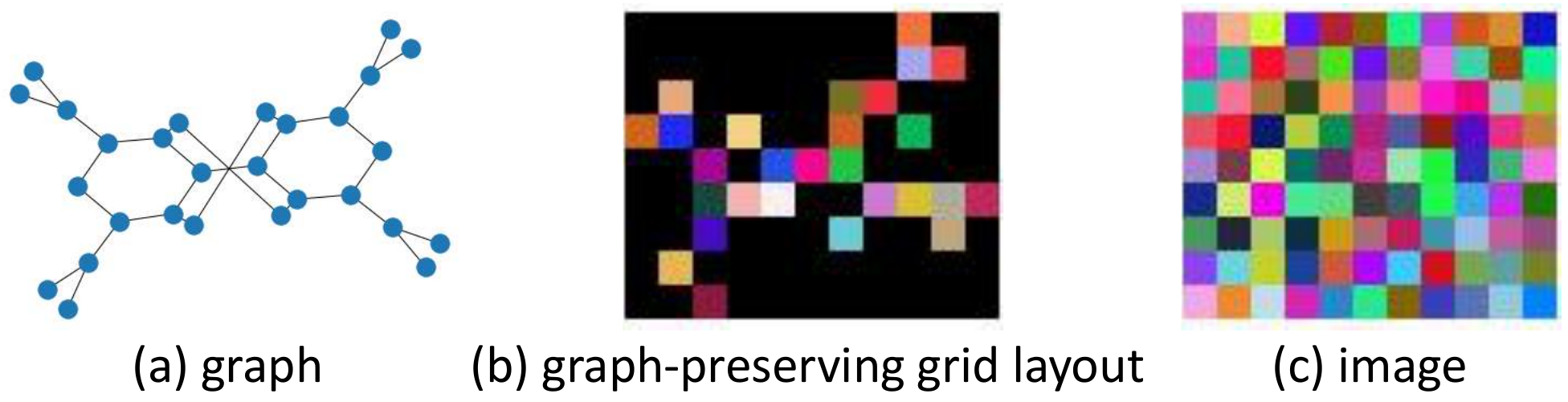}
	\caption{\footnotesize Illustration of differences between {\bf (a)} a graph, {\bf (b)} a graph-preserving grid layout (GPGL) of the graph, and {\bf (c)} an image. The black color in (b) denotes no correspondence to any vertex in (a), and other colors denote non-zero features on the grid vertices. 
	}
	\label{fig:overview}
	\vspace{-3mm}
\end{figure}

\bfsection{Motivation}
The key differences between graph convolution and 2D spatial convolution make CNNs much easier and richer in deep architectures than GCNs as well as being trained much more efficiently. Most of the previous works in the literature such as \cite{morris2019weisfeiler, niepert2016learning, ivanov2018anonymous} concentrate on handling the convolution weights based on the node connections, leading to high computation. Intuitively there also exists another way to process each graph by projecting the graph nodes onto a 2D grid so that CNNs can be employed directly. 
Such a method can easily benefit from CNNs in terms of network design and training in large scale. However, there are few existing works to explore this type of approach. As graph topology can be richer and more flexible than grid for encoding important messages of the graph, how to preserve graph topology on the grid becomes a key challenge in algorithm development.

Therefore, in order to bridge the gap between GCNs and CNNs, in contrast to previous works on generalizing the basic operations in CNNs to graph inputs, in this paper we mainly focus on studying the problem of \textbf{how to use CNNs as backbone for graph-based applications effectively and efficiently}. We then propose a principled method for projecting undirected graphs onto the 2D grid with graph topology preservation.

In fact, the visualization of graphs in 2D space has been well studied in {\em graph drawing}, an area of mathematics and computer sciences whose goal is to present the nodes and edges of a graph on a plane with some specific properties (\eg minimizing edge crossings \cite{chrobak1995linear,schnyder1990embedding}, minimizing the graph-node distance between graph domain and 2D domain \cite{kamada1989algorithm,kobourov2012spring}, showing possible clusters among the nodes \cite{frishman2007multi}). In the literature, the {\em Kamada-Kawai (KK)} algorithm \cite{kamada1989algorithm} is one of the most widely-used undirected graph visualization techniques. In general, the KK algorithm defines an objective function that measures the energy of each graph layout \wrt some theoretical graph distance, and searches for the (local) minimum that gives a reasonably good 2D visualization of the graph regardless the distances among the nodes. 

As described above, we can see that graph drawing algorithms may play an important role in connecting graph applications with CNNs for geometric deep learning (GDL), in general. To the best of our knowledge, however, such graph drawing algorithms have never been explored for GDL. One possible reason is that graph drawing algorithms often work in continuous spaces, while our case requires {\em discrete} spaces (\ie grid) where CNNs can be deployed. Overall, how to project graphs onto the grid with topology preservation for GDL is still elusive in the literature.

\bfsection{Contributions}
To address the problem above, in this paper we propose a novel {\em graph-preserving grid layout} (GPGL), an integer programming problem that minimizes the topological loss on the 2D grid so that CNNs can be used for GDL on undirected graphs. Technically solving such a problem is very challenging because potentially one needs to solve a highly nonconvex optimization problem in a discrete space. We manage to do so effectively by proposing a penalized KK method with a novel vertex separation penalty, followed by the rounding technique. As a result, our GPGL algorithm can approximately preserve the irregular structural information in a graph on the regular grid as graph layout, as illustrated in Fig. \ref{fig:overview}. To further improve the computational efficiency of GPGL, we also propose a hierarchical GPGL (H-GPGL) algorithm as an extension to handle large graphs.

In summary, our key contributions of this paper are as follows:
\setlist[itemize]{leftmargin=*}
\begin{itemize}
	\item We are the {\em first}, to the best of our knowledge, to explicitly explore the usage of graph drawing algorithms in the context of GDL, and accordingly propose a novel GPGL algorithm and its variance H-GPGL to project graphs onto the 2D grid with minimum loss in topological information.
	\item We demonstrate the empirical success of GPGL on graph classification with small graphs, and H-GPGL on 3D point cloud segmentation with large graphs. In the experiment, we clearly shows that PointNet with 2D max-pooling layers and 2D CNNs including VGG16, ResNet50 and multi-scale maxout(MSM) CNN benefit from the GPGL image representation and gain a large improvement over PointNet, a point-wised neural network.
\end{itemize}

\section{Related Work}
\bfsection{Graph Drawing \& Network Embedding}
Graph drawing can be considered as a subdiscipline of network embedding \cite{hamilton2017representation, cui2018survey, cai2018comprehensive} whose goal is to find a low dimensional representation of the network nodes in some metric space so that the given similarity (or distance) function is preserved as much as possible. In summary, graph drawing focuses on the 2D/3D visualization of graphs \cite{dougrusoz1996circular, eiglsperger2001orthogonal, koren2005drawing, spielman2007spectral, tamassia2013handbook}, while network embedding emphasizes the learning of low dimensional graph representations. Despite the research goal, similar methodology has been applied to both areas. For instance, the KK algorithm \cite{kamada1989algorithm} was proposed for graph visualization as a force-based layout system with advantages such as good-quality results and strong theoretical foundations, but suffering from high computational cost and poor local minima. Similarly \cite{tenenbaum2000global} proposed a global geometric framework for network embedding to preserve the intrinsic geometry of the data as captured in the geodesic manifold distances between all pairs of data points. There are also some works on drawing graphs on lattice, \eg \cite{freese2004automated}.

In contrast to graph drawing, our focus is to project an existing graph onto the grid with minimum topological loss so that CNNs can be deployed efficiently and effectively to handle graph data. In such a context, we are not aware of any work in the literature that utilizes the graph drawing algorithms to facilitate GDL, to the best of our knowledge.

\bfsection{Graph Synthesis \& Generation}
Methods in this field, \eg \cite{grover2018graphite, li2018learning, you2018graphrnn, samanta2019nevae}, often aim to learn a (sophisticated) generative model that reflects the properties of the training graphs. Recently, \cite{kwon2019deep} proposed learning an encoder-decoder for the graph layout generation problem to systematically visualize a graph in diverse layouts using deep generative model. \cite{franceschi2019learning} proposed jointly learning the graph structure and the parameters of GCNs by approximately solving a bilevel program that learns a discrete probability distribution on the edges of the graph for classification problems.

In contrast to such methods above, our algorithm for GPGL is essentially a self-supervised learning algorithm that is performed for each individual graph and requires no training at all. Moreover, we focus on re-deploying each graph onto the grid as layout while preserving its topology. This procedure is separate from the training of CNNs later.

\begin{figure*}[t]
    \hfill
	\begin{minipage}[b]{0.195\textwidth}
		\begin{center}
			\centerline{\includegraphics[width=\columnwidth]{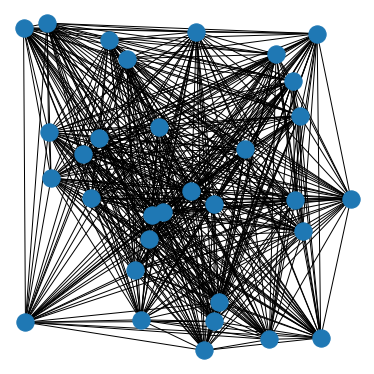}}
			\centerline{\footnotesize (a) Original graph}
		\end{center}
	\end{minipage}
    \hfill
	\begin{minipage}[b]{0.195\textwidth}
		\begin{center}
			\centerline{\includegraphics[width=\columnwidth]{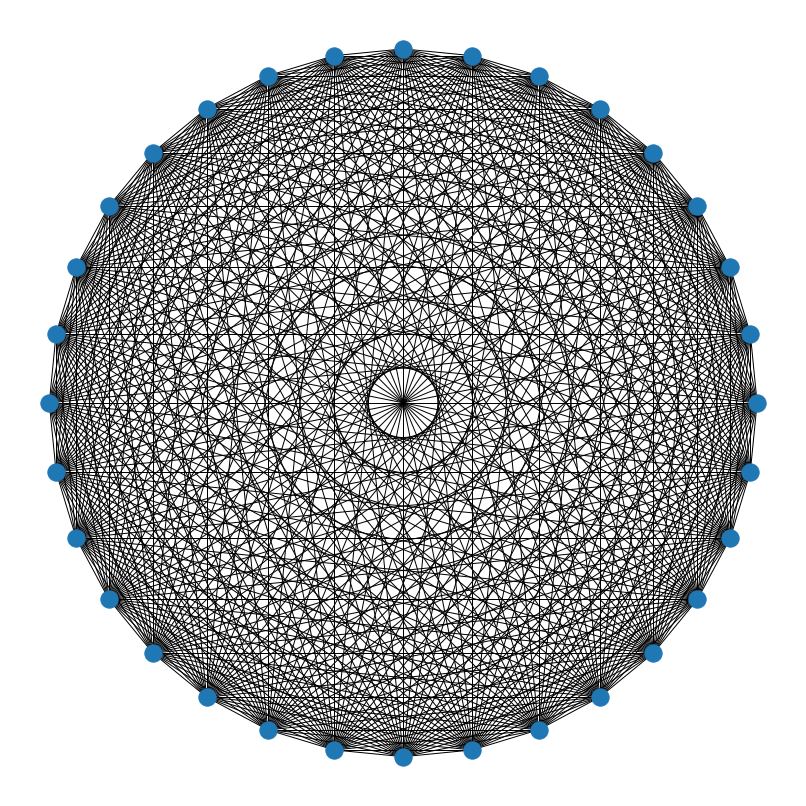}}
			\centerline{\footnotesize (b) KK before rounding}
		\end{center}
	\end{minipage}
	\hfill
	\begin{minipage}[b]{0.195\textwidth}
		\begin{center}
			\centerline{\includegraphics[width=\columnwidth]{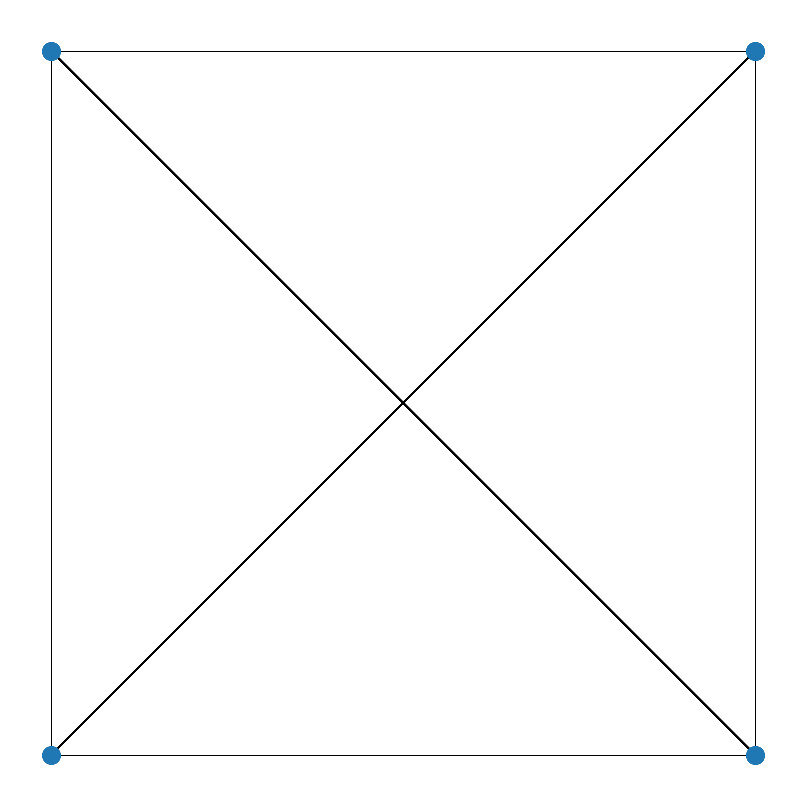}}
			\centerline{\footnotesize (c) KK after rounding}
		\end{center}
	\end{minipage}
	\hfill
	\begin{minipage}[b]{0.195\textwidth}
		\begin{center}
			\centerline{\includegraphics[width=\columnwidth]{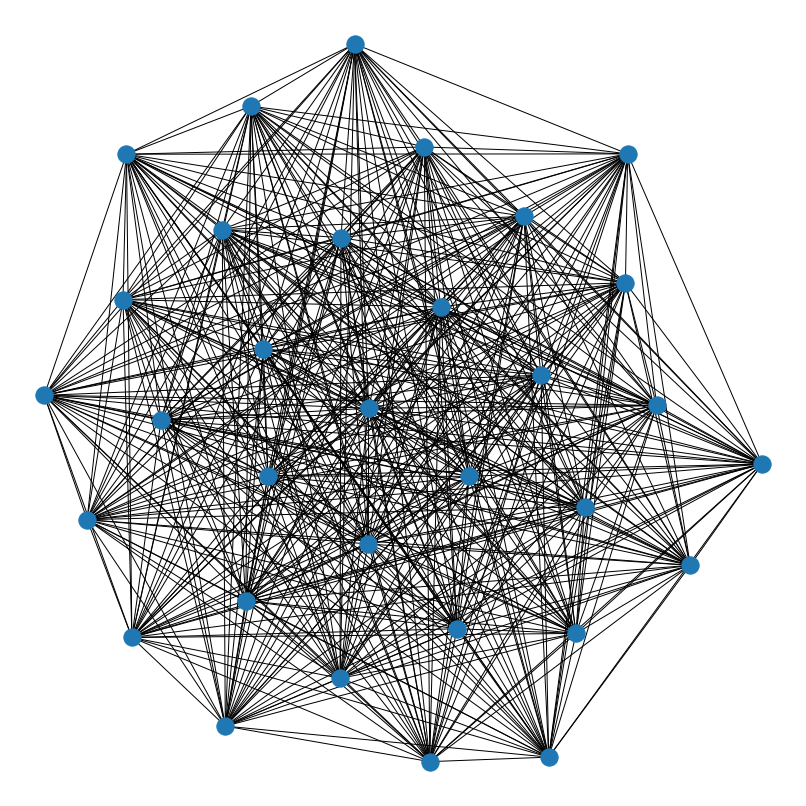}}
			\centerline{\footnotesize (d) Ours before rounding}
		\end{center}
	\end{minipage}
	\hfill
	\begin{minipage}[b]{0.195\textwidth}
		\begin{center}
			\centerline{\includegraphics[width=\columnwidth]{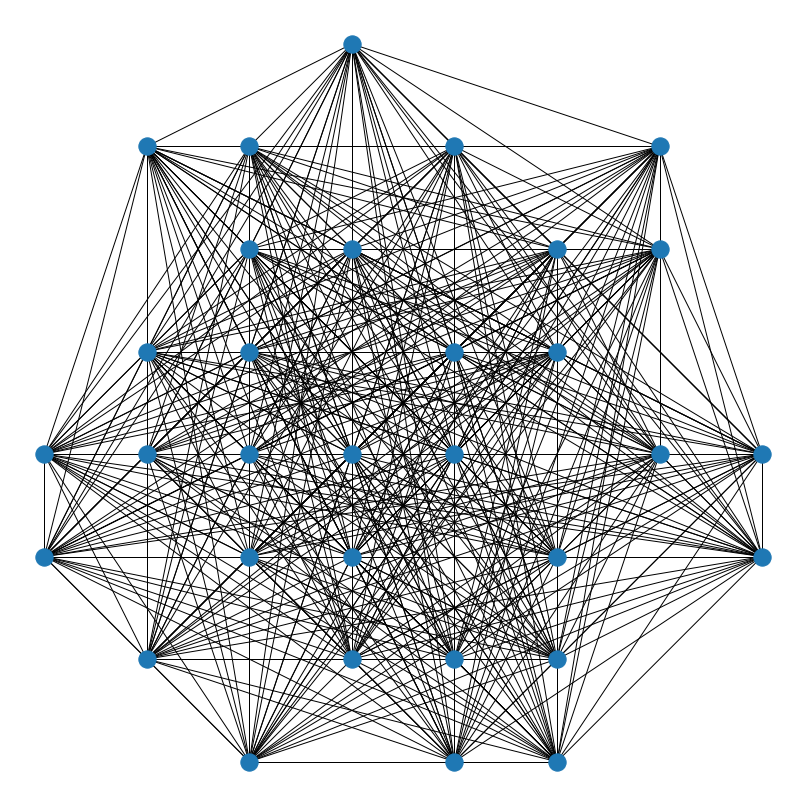}}
			\centerline{\footnotesize (e) Ours after rounding}
		\end{center}
	\end{minipage}
	\hfill
	\vspace{-5mm}
	\caption{\footnotesize Illustration of layout comparison between the KK algorithm and our proposed penalized KK algorithm before and after rounding based on a fully-connected graph with 32 vertices.}
	\label{fig:drawing_comparison}
	\vspace{-3mm}
\end{figure*}

\bfsection{Geometric Deep Learning}
In general GDL studies the extension of deep learning techniques to graph and manifold structured data (\eg \cite{kipf2016semi, bronstein2017geometric, monti2017geometric, huang2018building}). In particular in this paper we focus on graph data only. Broadly GDL methods can be categorized into spatial methods (\eg \cite{masci2015geodesic, boscaini2016learning, monti2017geometric}) and spectral methods (\eg \cite{defferrard2016convolutional, levie2018cayleynets, yi2017syncspeccnn}). Some nice survey on this topic can be found in \cite{bronstein2017geometric, zhang2020deep, hamilton2017representation, wu2020comprehensive}.

\cite{niepert2016learning} proposed a framework for learning convolutional neural networks by applying the convolution operations to the locally connected regions from graphs. We are different from such work by applying CNNs to the grids where graphs are projected to with topology preservation.
\cite{tixier2019graph} proposed an ad-hoc method to project graphs onto 2D grid and utilized CNNs for graph classification. Specifically each node is embedded into a high dimensional space, then mapped to 2D space by PCA, and finally quantized into a grid. 
In \cite{heimann2019distribution}, graphs were represented in 2D by approximating the Laplacian kernel mean map. In \cite{lyu2021treernn}, spanning trees were utilized to project a graph onto the image domain. Ho \etal \cite{ho2020efficient} converted adjacency matrices into more Column Ordering Free (COF) matrices and then built a deep network to learn high level representations from these matrices for graph classification.
In contrast, we propose a principled and systematic way based on graph drawing, \ie a nonconvex integer programming formulation, for mapping graphs onto 2D grid. Besides, our graph classification performance is much better than both works. On MUTAG and IMDB-B datasets we can achieve 94.18\% and 74.9\% test accuracy with 5.23\% improvement over \cite{niepert2016learning} and 4.5\% improvement over \cite{tixier2019graph}, respectively.

\bfsection{3D Point Cloud Processing}
Point clouds can be treated as graphs with undetermined connections. In point cloud processing, several existing works such as \cite{zhang2019linked} built the adjacency matrix using the K-nearest-neighbor search and then applied GCNs to the generated graph. Some other works like \cite{te2018rgcnn} used adaptive search ranges to generate the adjacency matrix and then applied GCNs to it. Point cloud processing, same as graph processing, suffers from the limited size and efficiency of GCN operators.

3D point cloud segmentation, as one of the key applications using point clouds, has attracted increasing research attention in the recent past. The existing works can be summarized into two major groups: point-based approaches and graph-based approaches. In point-based approaches, the pioneering work PointNet \cite{qi2017pointnet} applied fully connected layers to each single point and then extracted the global features by aggregating all the point-wise features. By its design, PointNet effectively gathers the global features from all points and broadcasts them back to all point for point-wise semantic prediction. However, the accuracy is limited due to the lack of local feature extraction. Later on, some other solutions are proposed to introduce the local feature extractions. PointNet++ \cite{qi2017pointnet++}, the succeeding work of PointNet, hierarchically grouped the points in local regions and extracted the local features along the hierarchy. PointCNN \cite{li2018pointcnn}, SpiderCNN \cite{xu2018spidercnn}, So-Net \cite{li2018so}, RS-CNN \cite{liu2019relation}, PointConv \cite{wu2019pointconv}, $\mathcal{\psi}$-CNN \cite{lei2019octree}, A-CNN \cite{komarichev2019cnn}, ShellNet \cite{zhang2019shellnet}, DensePoint \cite{liu2019densepoint},SPLATNet3D\cite{su2018splatnet} proposed their local point grouping and feature extraction operators, resulting in high accuracy.

The graph-based approaches, on the other hand, construct a connected graph from each point cloud and then apply GCNs to it. Kd-Net \cite{klokov2017escape}, RGCNN \cite{te2018rgcnn}, FeaStNet \cite{verma2018feastnet}, KC-Net \cite{shen2018mining}, DGCNN \cite{wang2019dynamic}, LDGCNN \cite{zhang2019linked} are some good works in this branch. However, the complexity of graph neural convolution kernels limits their capability of achieving outstanding accuracy. 

Besides, there are several works trying to project the point clouds into 2D grid maps and apply 2D convolution to it. SFCNN \cite{rao2019spherical} projected each point cloud into a spherical surface and applied a fractal convolution operation to it. InterpConv \cite{mao2019interpolated} creatively interpolated the graph convolution into a 2D convolution upon the adjacent grid cells. Our recent work \cite{lyu2020learning} solves the point cloud segmentation problem by transferring the point cloud into image space and applying U-Net on it. Those three approaches make good effort to introduce the 2D convolution into point cloud processing. However, they yield customized 2D convolution kernels that do not support commonly used feature extraction backbones such as VGG \cite{Simonyan15VGG}, ResNet \cite{he2016deep} and Xception \cite{chollet2017xception}.

In our approach, we follow the graph-based approaches to construct a connected graph upon each point cloud. In contrast, we further project the graph nodes onto a 2D grid map using H-GPGL. In this way each point in the point cloud is assigned to one of the nodes in the 2D grid and we can now apply CNNs to the 2D graph layouts as we do on images.

\section{Graph-Preserving Grid Layout (GPGL)}\label{sec:GPGL}

\subsection{Problem Setup}
Let $\mathcal{G}=(\mathcal{V}, \mathcal{E})$ be an undirected graph with a vertex set $\mathcal{V}$ and an edge set $\mathcal{E}\subseteq\mathcal{V}\times\mathcal{V}$, and $s_{ij}\geq1, \forall i\neq j$ be the graph-theoretic distance such as shortest-path between two vertices $v_i, v_j\in\mathcal{V}$ on the graph that encodes the graph topology. 

Now we would like to learn a function $f:\mathcal{V}\rightarrow\mathbb{Z}^2$ to map the graph vertex set to a set of 2D integer coordinates on the grid so that the graph topology can be preserved as much as possible given a metric $d:\mathbb{R}^2\times\mathbb{R}^2\rightarrow\mathbb{R}$ and a loss $\ell:\mathbb{R}\times\mathbb{R}\rightarrow\mathbb{R}$. As a result, we are seeking for $f$ to minimize the following objective:
\begin{align}\label{eqn:general_formula}
\min_{f}\sum_{i\neq j}\ell(d(f(v_i), f(v_j)), s_{ij}).
\end{align}
Now letting $\mathbf{x}_i=f(v_i)\in\mathbb{Z}^2$ as reparametrization, we can rewrite Eq. \ref{eqn:general_formula} as the following {\em integer programming} problem:
\begin{align}\label{eqn:int_pro}
\min_{\mathcal{X}\subseteq\mathbb{Z}^2}\sum_{i\neq j}\ell(d(\mathbf{x}_i, \mathbf{x}_j), s_{ij}),
\end{align}
where the set $\mathcal{X}=\{\mathbf{x}_i\}$ denotes the {\em 2D grid layout} of the graph, \ie all the vertex coordinates on the 2D grid. 

\bfsection{Self-Supervision}
Note that the problem in Eq. \ref{eqn:int_pro} needs to be solved for each individual graph, which is related to self-supervision as a form of unsupervised learning where the data itself provides the supervision \cite{zisserman2018}. This property is beneficial for data augmentation, as every local minimum will lead to a grid layout for the same graph.

\bfsection{2D Grid Layout}
In this paper we are interested in learning only 2D grid layouts for graphs, rather than higher dimensional grids (even 3D) where we expect that the layouts would be more compact in volume and would have larger variance in configuration, both bringing more challenges into training CNNs properly later. We confirm our hypothesis based on empirical observations. Besides, the implementation of 3D basic operations in CNNs such as convolution and pooling are often slower than 2D counterparts, and the operations beyond 3D are not available publicly. 

\bfsection{Relaxation \& Rounding for Integer Programming}
Integer programming is NP-complete and thus finding exact solutions is challenging, in general \cite{wolsey2014integer}. Relaxation and rounding is a widely used heuristic for solving integer programming due to its efficiency \cite{bradley1977applied}, where the rounding operator is applied to the solution from the real-number relaxed problem as the solution for the integer programming. In this paper we employ this heuristic to learn 2D grid layouts. For simplicity, in the sequel we will only discuss how to solve the relaxation problem (\ie before rounding).

\subsection{Penalized Kamada-Kawai Algorithm for GPGL}
In this paper we set $\ell$ and $d$ in Eq. \ref{eqn:int_pro} to the least-square loss and Euclidean distance to preserve topology, respectively, so that we can develop new algorithms based on the classic KK algorithm. 

\subsubsection{Preliminary: Kamada-Kawai Algorithm}
The KK graph drawing algorithm \cite{kamada1989algorithm} was designed for a (relaxed) problem in Eq. \ref{eqn:int_pro} with a specific objective function as follows:
\begin{align}\label{eqn:kk}
\min_{\mathcal{X}\subseteq\mathbb{R}^2}\mathcal{L}_{KK} = \sum_{i\neq j}\frac{1}{2}\left(\frac{d_{ij}}{s_{ij}} - 1\right)^2,
\end{align}
where $d_{ij}=\|\mathbf{x}_i - \mathbf{x}_j\|, \forall (i,j)$ denotes the Euclidean distance between vertices $v_i$ and $v_j$. Note that there is no regularization/penalty to control the distribution of nodes in 2D visualization. 

Fig. \ref{fig:drawing_comparison} illustrates the problems using the KK algorithm when projecting the fully-connected graph onto 2D grid. Eventually KK learns a circular distribution with equal space among the vertices as in Fig. \ref{fig:drawing_comparison}(a) to minimize the topology preserving loss in Eq.~\ref{eqn:kk}. When taking a close look at these 2D locations we find that after transformation all these locations are within the square area $[0,1]\times[0,1]$, leading to the square pattern in Fig. \ref{fig:drawing_comparison}(b) after rounding. Such behavior totally makes sense to KK because it does not care about the grid layout but only the topology preserving loss. However, our goal is not only to preserve the topology but also to make graphs visible on the 2D grid in terms of vertices.

\subsubsection{Our Algorithm}
To this end, we propose a penalized KK algorithm as listed in Alg.~\ref{alg:RKK_alg}, which tries to minimize the following penalized KK loss:
\begin{align}\label{eqn:L_GPGL}
    \min_{\mathcal{\mathcal{X}}\subseteq\mathbb{Z}^2}\mathcal{L}_{GPGL} = \mathcal{L}_{KK} + \mathcal{L}_{sep}.
\end{align}

\bfsection{Vertex Separation Penalty}
We propose a novel vertex separation penalty to regularize the vertex distribution on the grid. The intuition behind it is that when the minimum distance among all the vertex pairs is larger than a threshold, say 1, it will guarantee that after rounding every vertex will be mapped to a unique 2D location with no overlap. But when any distance is smaller than the threshold, it should be considered to enlarge the distance, otherwise, no penalty. Moreover, we expect that the penalties will grow faster than the change of distances and in such a way the vertices can be re-distributed more rapidly. Based on these considerations we propose the following penalty:
\begin{align}\label{eqn:r_sep}
\mathcal{L}_{sep} = \lambda\sum_{i\neq j} \max\left\{0, \frac{\alpha}{d_{ij}} - 1\right\}, 
\end{align}
where $\alpha\geq0, \lambda\geq0$ are two predefined constants. From the gradient of $\mathcal{L}_{sep}$ \wrt an arbitrary 2D variable $\mathbf{x}_i$, that is,
\begin{align}
\frac{\partial\mathcal{L}_{sep}}{\partial\mathbf{x}_i} = -\lambda\sum_{i\neq j} \frac{\mathbf{x}_i - \mathbf{x}_j}{d_{ij}^3}\mathbf{1}_{\{d_{ij}<\alpha\}}
\end{align}
where $\mathbf{1}_{\{\cdot\}}$ denotes the indicator function returning 1 if the condition is true, otherwise 0, we can clearly see that $\alpha$ as a threshold controls when penalties occur, and $\lambda$ controls the trade-off between the two losses, leading to different step sizes in gradient based optimization.

\begin{algorithm}[t]
	\SetAlgoLined
	\SetKwInOut{Input}{Input}\SetKwInOut{Output}{Output}
	\Input{undirected graph $\mathcal{G}=(\mathcal{V}, \mathcal{E})$, parameters $\alpha, \lambda$}
	\Output{2D grid layout $\mathcal{X}^*$}
	\BlankLine
	
	Compute graph distance $\{s_{ij}\}$; 
	
	$\tilde{\mathcal{X}}\leftarrow\argmin_{\mathcal{X}}\mathcal{L}_{KK}$ with a (randomly shuffled) circular layout;
	
	
	
	
	$\mathcal{X}^*\leftarrow\argmin_{\mathcal{X}\subseteq\mathbb{R}^2}\mathcal{L}_{GPGL}$ with set $\tilde{\mathcal{X}}$ as initialization;
	
	$\mathcal{X}^*\leftarrow\mbox{round}(\mathcal{X}^*)$;
	
	\Return $\mathcal{X}^*$;
	\caption{Penalized Kamada-Kawai Algorithm for GPGL}\label{alg:RKK_alg}
\end{algorithm}

\bfsection{Initialization}
Note that both KK and our algorithms are highly nonconvex, and thus good initialization is need to make both work well, \ie convergence to good local minima. 

To this end, we first utilize the KK algorithm to generate a vertex distribution. To do so, we employ the implementation in the Python library {\sc NetworkX} \cite{team2014networkx} which uses a circular layout as initialization by default. By default setting, L-BFGS-B Nonlinear Optimization is implemented to optimize Eqn. \ref{eqn:L_GPGL}. As discussed above, KK has no control on the vertex distribution. This may lead to serious {\bf vertex loss} problems in the 2D grid layout where some of vertices in the original graph merge together as a single vertex on the grid after rounding due to small distances (see our experiments).



\bfsection{Topology Preservation with Penalty}
As we observe, the key challenge in topology preservation comes from the node degree, and the lower degree the easier for preservation. Since there are only 8 neighbors at most in the 2D grid layout, it will induce a penalty for a graph vertex whose degree is higher than 8. Fig. \ref{fig:drawing_comparison} illustrates such a case where the original graph is full-connected with 32 vertices. With the help of our proposed penalty, we manage to map this graph to a ball-like grid layout, as shown in Fig. \ref{fig:drawing_comparison}(c) and (d). 
Besides we have the following proposition to support such observations:

\begin{prop}\label{prop:1}
	An ideal 2D grid layout with no vertex loss for a full-connected graph with $|\mathcal{V}|$ vertices is a ball-like shape with radius of $\lceil (\frac{|\mathcal{V}|}{\pi})^{\frac{1}{2}}\rceil$ that minimizes Eq.~\ref{eqn:int_pro} with relaxation of the penalized Kamada-Kawai loss. Here $\lceil\cdot\rceil$ denotes the ceiling operation.
\end{prop}
\begin{proof}
	Given the conditions in the proposition above, we have $s_{ij}=1, d_{ij}\geq 1,  \forall i\neq j$ and $\mathcal{L}_{sep}=0$. Without loss of generalization, we uniformly deploy the graph vertices in a circle and set the circular center $A$ to a node on the 2D grid. Now imagine the gradient field over all the vertices as a sandglass centered at $A$ where each vertex is a ball with a unit diameter. Then it is easy to see that by the ``gravity'' (\ie gradient) all the vertices move towards the center $A$, and eventually are stabilized (as a local minimum) within an $r$-radius circle whose covering area should satisfy $|\mathcal{V}|\leq \pi r^2$, \ie $r=\lceil (\frac{|\mathcal{V}|}{\pi})^{\frac{1}{2}}\rceil$ as the smallest sufficient radius to cover all the vertices with guarantee. We now complete our proof.
\end{proof}

Note that Fig. \ref{fig:drawing_comparison}(d) exactly verifies Prop. \ref{prop:1} with a radius $r=\lceil (\frac{32}{\pi})^{\frac{1}{2}}\rceil=4$. In summary, our algorithm can (approximately) manage to preserve graph topology on the 2D grid even when the node degree is higher than 8.

\bfsection{Computational Complexity} 
The KK algorithm has the complexity of, at least, $O(|\mathcal{V}|^2)$ \cite{kobourov2012spring} that limits the usage of KK to medium-size graphs (\eg 50-500 vertices). Since our algorithm in Alg. \ref{alg:RKK_alg} is based on KK, it unfortunately inherits this limitation as well. To accelerate the computation for large-scale graphs, we potentially can adopt the strategy in multi-scale graph drawing algorithms such as \cite{harel2000fast}. However, such an extension is out of scope of this paper, and we will consider it in our future work.

\subsection{Experiments: Small Graph Classification}
\subsubsection{Data Sets \& Data Augmentation}
We evaluate our method, \ie GPGL + (multi-scale maxout) CNNs, on four medium-size benchmark data sets for graph classification, namely MUTAG, IMDB-B, IMDB-M and PROTEINS. Table \ref{tab:statistics} summarizes some statistics of each data set. Note that the max node degree on each data set is at least 8, indicating that ball-like patterns as discussed in Prop. \ref{prop:1} may occur, especially for IMDB-B and IMDB-M.

As mentioned in self-supervision, each local minimum from our penalized KK algorithm in Alg. \ref{alg:RKK_alg} will lead to a grid layout for the graph, while each minimum depends on its initialization. Therefore, to augment grid layout data from graphs, we do a random shuffle on the circular layout when applying Alg. \ref{alg:RKK_alg} to an individual graph. Fig. \ref{fig:ablation_augmentation} illustrates two augmented image representations from a graph sample in MUTAG dataset.

\begin{figure}[h]

	\begin{minipage}[b]{0.3\columnwidth}
		\begin{center}
			\centerline{\includegraphics[width=\columnwidth]{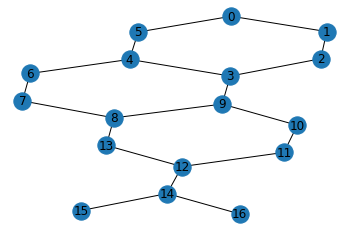}}
			\centerline{\footnotesize (a) Input graph}
		\end{center}
	\end{minipage}
	\hfill
	\begin{minipage}[b]{0.3\columnwidth}
		\begin{center}
			\centerline{\includegraphics[width=\columnwidth]{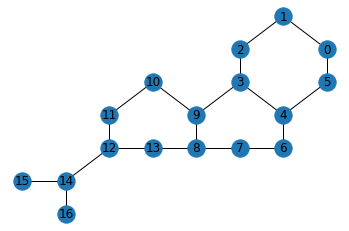}}
			\centerline{\footnotesize (b) Output image 1}
		\end{center}
	\end{minipage}
	\hfill
	\begin{minipage}[b]{0.3\columnwidth}
		\begin{center}
			\centerline{\includegraphics[width=\columnwidth]{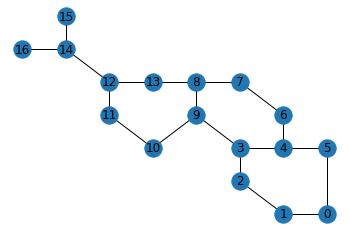}}
			\centerline{\footnotesize (c) Output image 2}
		\end{center}
	\end{minipage}
	\hfill
	\vspace{-3mm}
	\caption{\footnotesize Illustration of augmented representations from a graph using GPGL}
	\label{fig:ablation_augmentation}
\end{figure}
\vspace{-5mm}

\subsubsection{Grid-Layout based 3D Representation}
Once a grid layout is generated, we first crop the layout with a sufficiently large fixed-size window (\eg $32\times 32$), and then associate each vertex feature vector from the graph with the projected node within the window. All the layouts are aligned to the top-left corner of the window. The rest of nodes in the window with no association of feature vectors are assigned to zero vectors. 

Once vertex loss occurs, we take an average, by default, of all the vertex feature vectors (\ie average-pooling) and assign it to the grid node. We also compare average-pooling with max-pooling for merging vertices, and observe similar performance empirically in terms of classification.

\begin{figure*}[t]

	\begin{minipage}[b]{0.195\textwidth}
		\begin{center}
			\centerline{\includegraphics[width=\columnwidth]{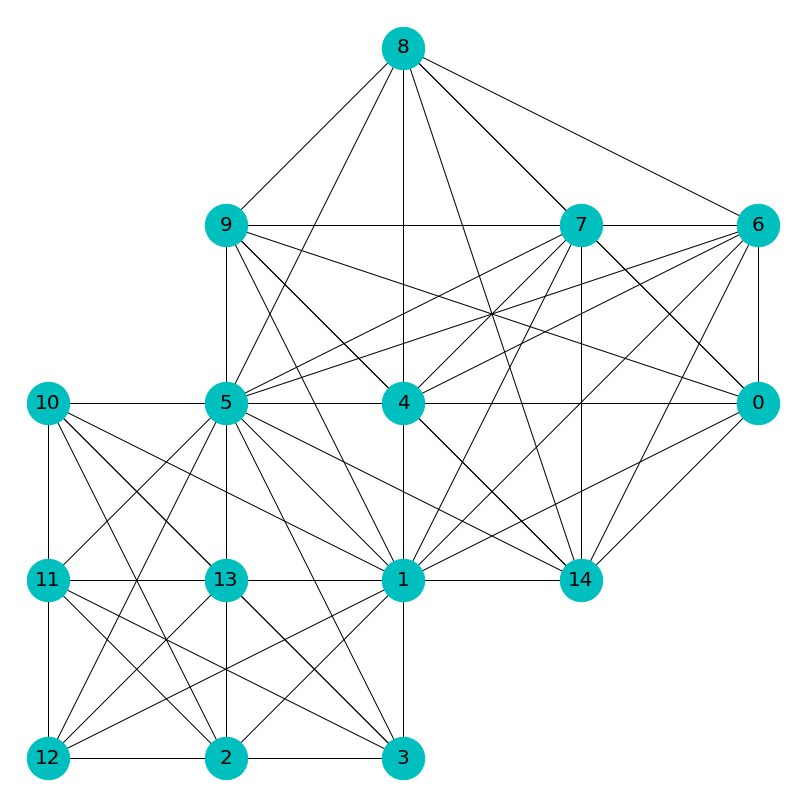}}
			\centerline{\footnotesize (a) $\alpha=1.00, \lambda=1000$}
		\end{center}
	\end{minipage}
	\hfill
	\begin{minipage}[b]{0.195\textwidth}
		\begin{center}
			\centerline{\includegraphics[width=\columnwidth]{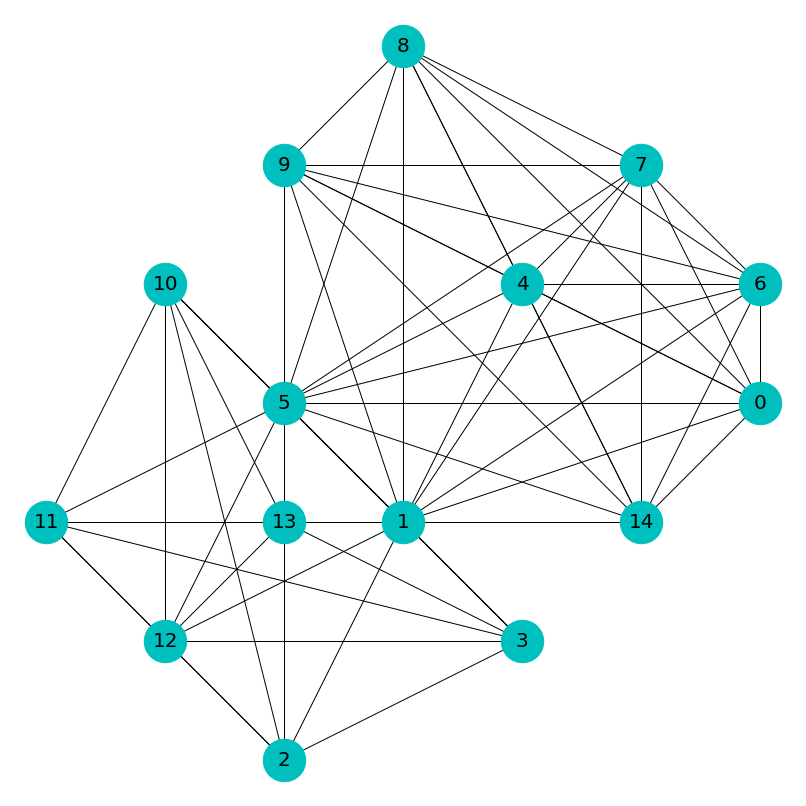}}
			\centerline{\footnotesize (b) $\alpha=1.50, \lambda=1000$}
		\end{center}
	\end{minipage}
	\hfill
	\begin{minipage}[b]{0.195\textwidth}
		\begin{center}
			\centerline{\includegraphics[width=\columnwidth]{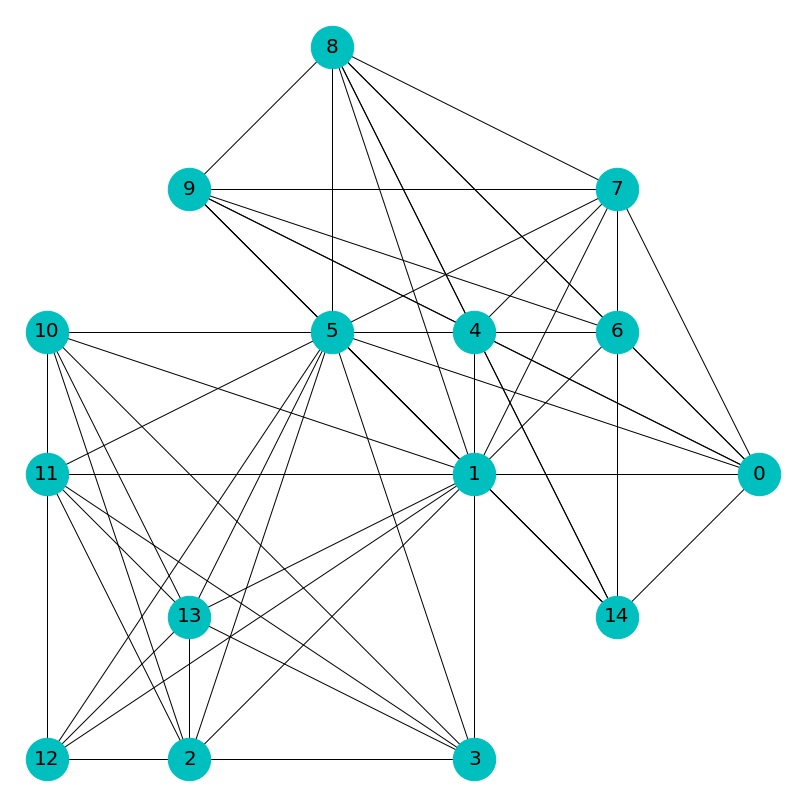}}
			\centerline{\footnotesize (c) $\alpha=1.25, \lambda=1000$}
		\end{center}
	\end{minipage}
	\hfill
	\begin{minipage}[b]{0.195\textwidth}
		\begin{center}
			\centerline{\includegraphics[width=\columnwidth]{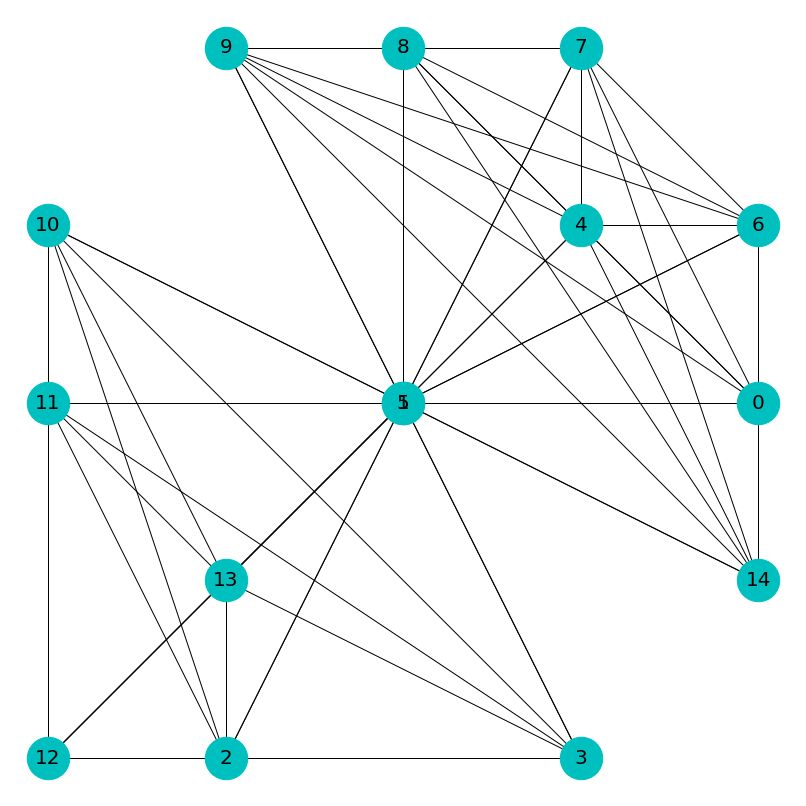}}
			\centerline{\footnotesize (d) $\alpha=1.25, \lambda=200$}
		\end{center}
	\end{minipage}
	\hfill
	\begin{minipage}[b]{0.195\textwidth}
		\begin{center}
			\centerline{\includegraphics[width=\columnwidth]{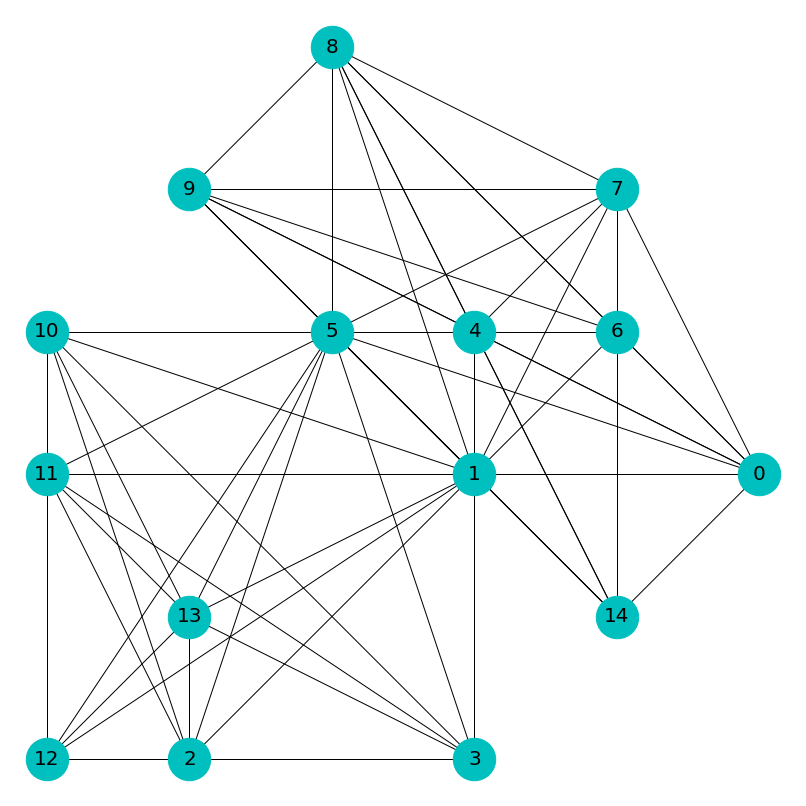}}
			\centerline{\footnotesize (e) $\alpha=1.25, \lambda=5000$}
		\end{center}
	\end{minipage}
	\hfill
	\vspace{-3mm}
	\caption{\footnotesize Illustration of effects of different combinations of $\alpha, \lambda$ on grid layout generation (IMDB-B)}
	\label{fig:alpha_lambda}
\end{figure*}
\begin{figure*}[t]
	\begin{minipage}[b]{0.3\textwidth}
		\begin{center}
			\centerline{\includegraphics[width=\columnwidth]{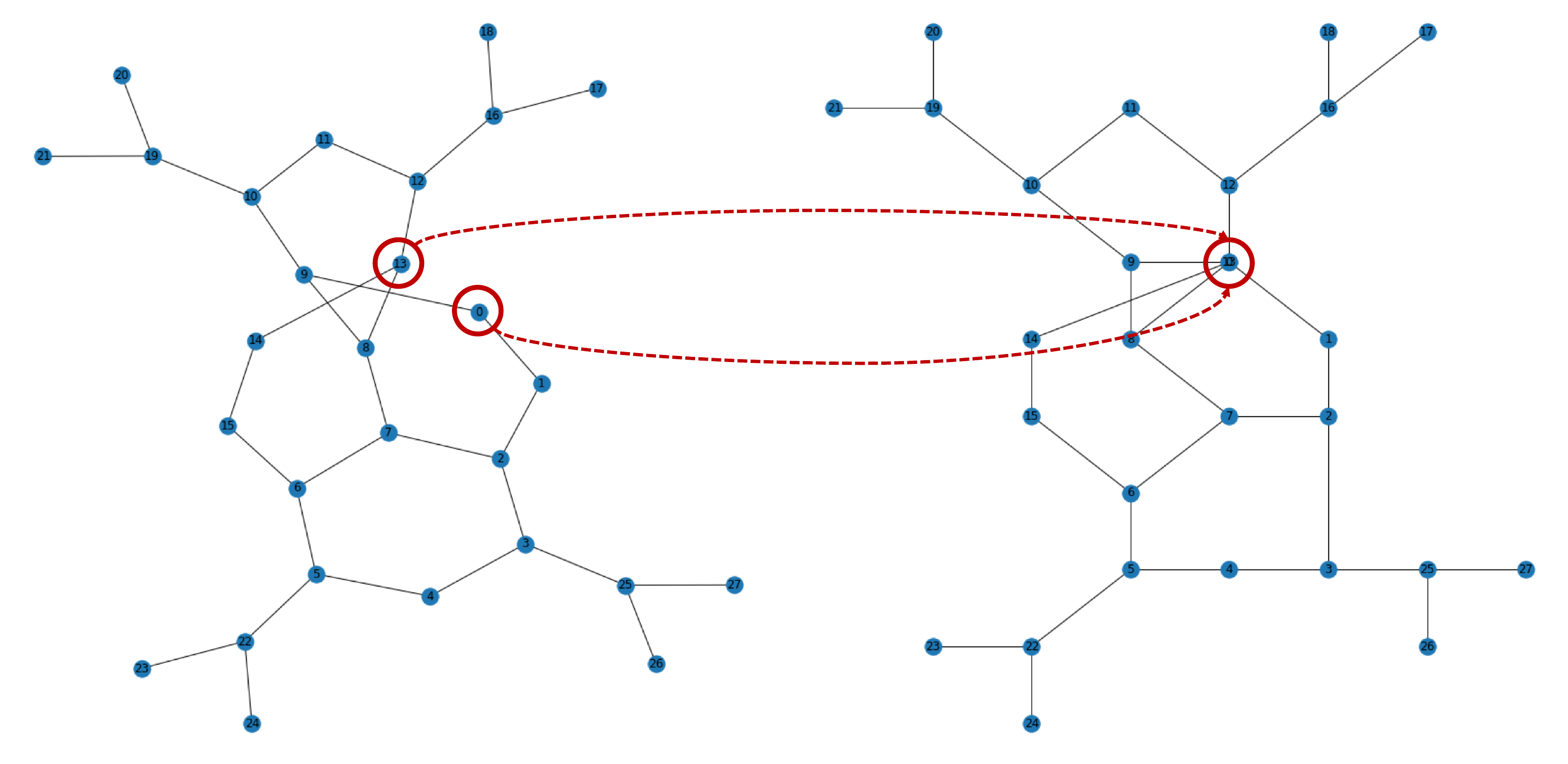}}
			\centerline{\footnotesize (a) MUTAG}
		\end{center}
	\end{minipage}
	\hfill
	\begin{minipage}[b]{0.3\textwidth}
		\begin{center}
			\centerline{\includegraphics[width=\columnwidth]{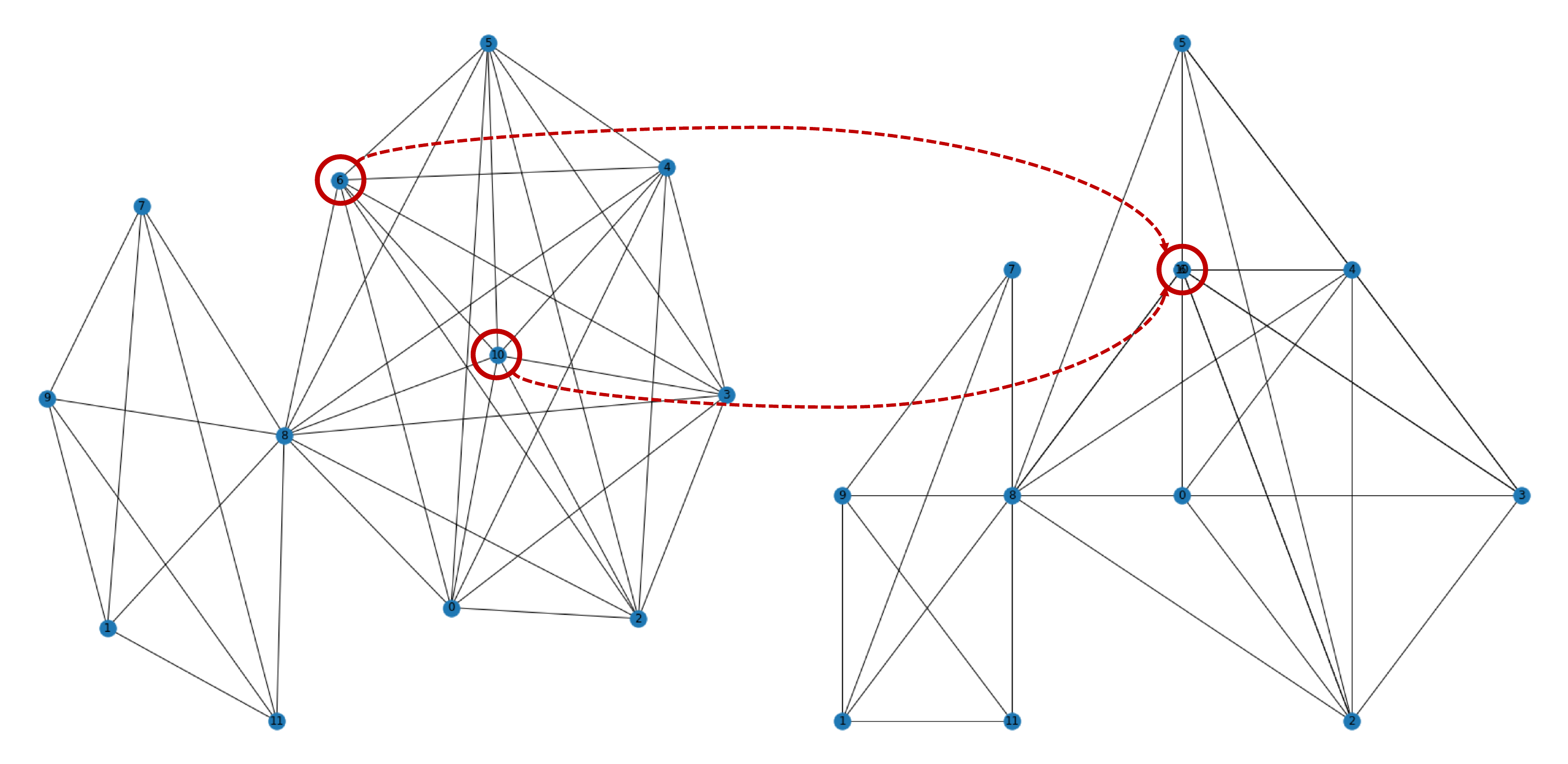}}
			\centerline{\footnotesize (b) IMDB-B}
		\end{center}
	\end{minipage}
	\hfill
	\begin{minipage}[b]{0.3\textwidth}
		\begin{center}
			\centerline{\includegraphics[width=\columnwidth]{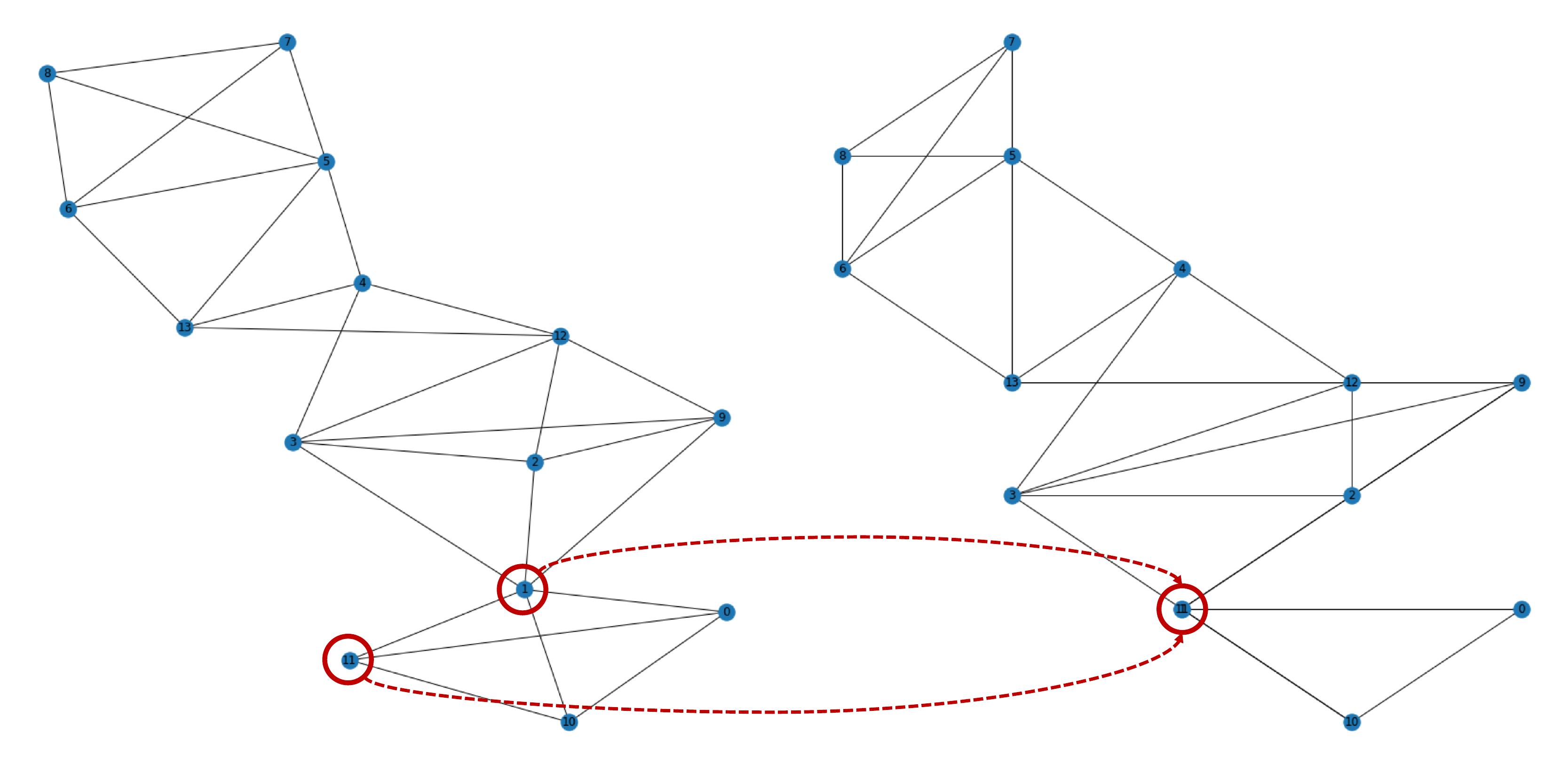}}
			\centerline{\footnotesize (c) PROTEINS}
		\end{center}
	\end{minipage}
	\hfill
	\caption{\footnotesize Illustration of vertex loss on different data sets: In each subfigure, {\bf (left)} before rounding and {\bf (right)} after rounding.}
	\label{fig:vertex_loss}
\end{figure*}

\subsubsection{Classifier: Multi-Scale Maxout CNNs (MSM-CNNs)}
We apply CNNs to the 3D representations of graphs for classification. As we discussed above, once the node degree is higher than 8, the grid layout cannot fully preserve the topology, but rather tends to form a ball-like compact pattern with larger neighborhood. To capture such neighborhood information effectively, the kernel sizes in the 2D convolution need to vary. Therefore, the problem now boils down to a feature selection problem with convolutional kernels.

Considering these, here we propose using a multi-scale maxout CNN as illustrated in Fig. \ref{fig:mm_cnn}. We use consecutive convolutions with smaller kernels to approximate the convolutions with larger kernels. For instance, we use three $3\times3$ kernels to approximate a $7\times7$ kernel. The maxout \cite{goodfellow2013maxout} operation selects which scale per grid node is good for classification and outputs the corresponding features. Together with other CNN operations such as max-pooling, we can design deep networks, if necessary.

\begin{figure}[H]
	\centering
	\includegraphics[width=0.95\linewidth]{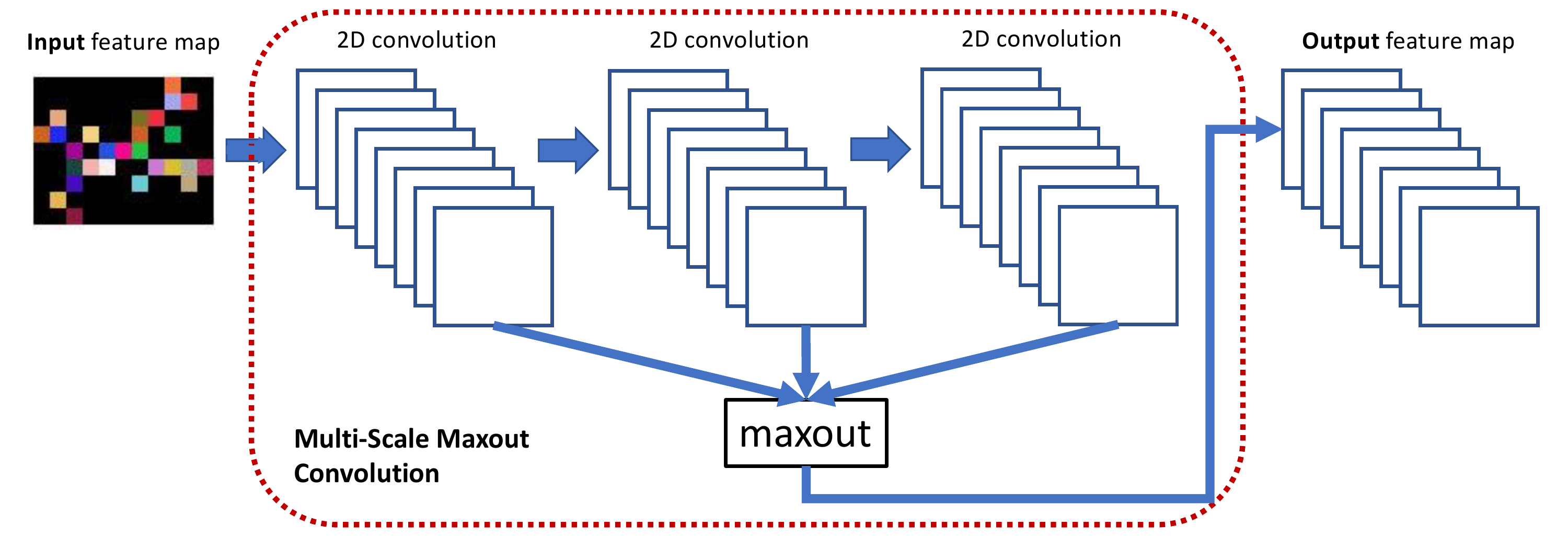}
	\caption{\footnotesize Multi-scale maxout convolution (MSM-Conv).
	}
	\label{fig:mm_cnn}
	\vspace{-3mm}
\end{figure}

\begin{table}[t]
	\caption{\footnotesize Statistics of benchmark data sets for graph classification.}
	\label{tab:statistics}
	\vspace{-2mm}
	\adjustbox{width=1.0\columnwidth}{
		\centering
		\begin{tabular}{c|ccccccc}
			\toprule
			Data Set & \begin{tabular}[c]{@{}c@{}}Num. of\\Graph\end{tabular} & \begin{tabular}[c]{@{}c@{}}Num. of\\Class\end{tabular} & \begin{tabular}[c]{@{}c@{}}Avg.\\Node\end{tabular} & \begin{tabular}[c]{@{}c@{}}Avg.\\Edge\end{tabular} & \begin{tabular}[c]{@{}c@{}}Avg.\\Degree\end{tabular} & \begin{tabular}[c]{@{}c@{}}Max\\Degree\end{tabular} & \begin{tabular}[c]{@{}c@{}}Feat.\\Dim.\end{tabular} \\
			\midrule
			MUTAG & 188 & 2 & 17.93 & 19.79 & 1.10 & 8 & 7 \\
			IMDB-B & 1000 & 2 & 19.77 & 96.53 & 4.88 & 270 & 136 \\
			IMDB-M & 1500 & 3 & 13.00 & 65.94 & 5.07 & 176 & 89 \\
			PROTEINS & 1113 & 2 & 39.06 & 72.82 & 1.86 & 50 & 3 \\
			\bottomrule
		\end{tabular}
	}
	\vspace{-3mm}
\end{table}

\subsubsection{Implementation}
By default, we set the parameters in Alg. \ref{alg:RKK_alg} as $\alpha=1.25, \lambda=1000$. 
We crop all the grid layouts to a fixed-size $32\times 32$ window. Also by default, for the MSM-CNNs we utilize three consecutive $3\times3$ kernels in the MSM-Conv, and design a simple network of ``MSM-Conv(64)$\rightarrow$max-pooling$\rightarrow$MSM-Conv(128)$\rightarrow$max-pooling$\rightarrow$MSM-Conv(256)$\rightarrow$global-pooling$\rightarrow$FC(256)$\rightarrow$FC(128)'' as hidden layers with ReLU activations, where FC denotes a fully connected layer and the numbers in the brackets denote the numbers of channels in each layer. We employ Adam \cite{kingma2014adam} as our optimizer, and set batch size, learning rate, and dropout ratio to 10, 0.0001, and 0.3, respectively.

\subsubsection{Ablation Study}

\bfsection{Effects of $\alpha, \lambda$ on Grid Layout and Classification}
To understand their effects on the 2D grid layout generation, we visualize some results in Fig. \ref{fig:alpha_lambda} using different combinations of $\alpha, \lambda$. We can see that:
\begin{itemize}
	\item From Fig. \ref{fig:alpha_lambda}(a)-(c), the diameters of grid layouts are $5\times5$, $6\times6$, $7\times7$ for $\alpha=1.00, 1.25, 1.50$, respectively. This strongly indicates that a smaller $\alpha$ tends to lead to a more compact layout at the risk of losing vertices.
	\item From Fig. \ref{fig:alpha_lambda}(c)-(e), similarly the diameters of grid layouts are $5\times5$, $6\times6$, $6\times6$ for $\lambda=200, 1000, 5000$, respectively. This indicates that a smaller $\lambda$ tends to lead to a more compact layout at the risk of losing vertices as well. In fact in Fig. \ref{fig:alpha_lambda}(d) node 1 and node 5 are merged together. When $\lambda$ is sufficiently large, the layout tends to be stable.
\end{itemize}
Such observations follow our intuition in designing Alg. \ref{alg:RKK_alg}, and occur across all the four benchmark data sets.

We also test the effects on classification performance. For instance, we generate 21x grid layouts using data augmentation on MUTAG, and list our results in Table \ref{tab:alpha_lambda}. Clearly our default setting achieves the best test accuracy. We also observe that in case $\lambda=0$ where $\mathcal{L}_{sep}=0$ and the KK algorithm works without vertex separation penalty, the classification accuracy is much worse than those with vertex separation penalty.

\begin{table}[h]
	\caption{\footnotesize Mean accuracy (\%) using different combinations of $\alpha, \lambda$.}
	\label{tab:alpha_lambda}
	\vspace{-3mm}
	\adjustbox{width=\columnwidth}{
		\centering
		\begin{tabular}{c|cccccc}
			\toprule
			Data Set &
			\begin{tabular}[c]{@{}c@{}}$\alpha=1.00$\\$\lambda=0$\end{tabular}& \begin{tabular}[c]{@{}c@{}}$\alpha=1.00$\\$\lambda=1000$\end{tabular} & \begin{tabular}[c]{@{}c@{}}$\alpha=1.50$\\$\lambda=1000$\end{tabular} & \begin{tabular}[c]{@{}c@{}}$\alpha=1.25$\\$\lambda=1000$\end{tabular} & \begin{tabular}[c]{@{}c@{}}$\alpha=1.25$\\$\lambda=200$\end{tabular} & \begin{tabular}[c]{@{}c@{}}$\alpha=1.25$\\$\lambda=5000$\end{tabular} \\
			\midrule
			\begin{tabular}[c]{@{}c@{}}MUTAG\\($21\times$)\end{tabular} & 80.51 & 85.14 & 83.04 & 86.31 & 85.26 & 85.26 \\
			\bottomrule
		\end{tabular}
	}
	\vspace{-2mm}
\end{table}

\bfsection{Vertex Loss, Graph Topology \& Misclassification}
To better understand the problem of vertex loss, we visualize some cases in Fig. \ref{fig:vertex_loss}. The reason for this behavior is due to the small distances among the vertices returned by Alg. \ref{alg:RKK_alg} that cannot survive from rounding. Unfortunately we do not observe a pattern on when such loss will happen. Note that our Alg. \ref{alg:RKK_alg} cannot avoid vertex loss with guarantee, and in fact the vertex loss ratio on each data set is very low, as shown in Table \ref{tab:vertex_loss_ratio}.

\begin{table}[t]
	\caption{\footnotesize Vertex loss ratio (\%) on each data set using different initialization methods. Vertex loss ratio is the percentage between the number of overlapped vertices and the total number of vertices on average of each graph in the dataset.}
	\label{tab:vertex_loss_ratio}
	\vspace{-2mm}
	\adjustbox{width=\columnwidth}{
		\centering

		{\begin{tabular}{c|cccc}
		
			\toprule
			Initialization & MUTAG & IMDB-B & IMDB-M & PROTEINS
			\\
			\midrule
			Circular & 1.06 & 0.99 & 0.40 & 0.90 \\
			Spectral & 10.09 & 9.30 & 12.50 & 18.99 \\
			Random & 1.88 & 1.48 & 1.04 & 1.07 \\
			
			\bottomrule
		\end{tabular}}
	}
	\vspace{-1mm}
\end{table}

\begin{table}[t]
	\caption{\footnotesize Ratios (\%) between vertex loss and misclassification.}
	\label{tab:vertex_loss_miscls}
	\vspace{-2mm}
	\adjustbox{width=\columnwidth}{
		\centering
		\begin{tabular}{c|ccc}
			\toprule
			Data Set & \begin{tabular}[c]{@{}c@{}}$\underline{|\mathcal{G}_{v.l.}|}$\\$|\mathcal{G}_{mis.}|$\end{tabular} & \begin{tabular}[c]{@{}c@{}}$\underline{|\mathcal{G}_{v.l.}\cap\mathcal{G}_{mis.}|}$\\$|\mathcal{G}_{v.l.}|$\end{tabular} & \begin{tabular}[c]{@{}c@{}}$\underline{|\mathcal{G}_{n.v.l.}\cap\mathcal{G}_{mis.}|}$\\$|\mathcal{G}_{n.v.l.}|$\end{tabular}
			\\
			\midrule
			MUTAG (21x) & 1.06 & 20.00 & 16.70 \\
			IMDB-B (3x) & 0.99 & 16.18 & 37.41 \\
			PROTEINS (3x) & 0.90 & 24.32 & 29.89 \\
			\bottomrule
		\end{tabular}
	}
	\vspace{-3mm}
\end{table}

Further we test the relationship between vertex loss and misclassification, and list our results in Table \ref{tab:vertex_loss_miscls} where $\mathcal{G}_{v.l.}$, $\mathcal{G}_{n.v.l.}$, and $\mathcal{G}_{mis.}$ denote the sets of graphs with vertex loss, no vertex loss, and misclassification, respectively, $\cap$ denotes the intersection of two sets, $|\cdot|$ denotes the cardinality of the set, and the numbers in the brackets denote the numbers of grid layouts per graph in data augmentation. From this table, we can deduce that vertex loss cannot be the key reason for misclassification, because it takes only tiny portion in misclassification and the ratios of misclassified graphs with/without vertex loss are very similar, indicating that misclassification more depends on the classifier rather than vertex loss.

As discussed before, a larger node degree is more difficult for preserving topology. In this test we would like to verify whether such topology loss introduces misclassification. Compared with the statistics in Table \ref{tab:statistics}, it seems that topology loss does cause trouble in classification. One of the reasons may be that the variance of the grid layout for a vertex with larger node degree will be higher due to perturbation. Designing better CNNs will be one of our future works to improve the performance.

\bfsection{CNN based Classifier Comparison}
We test the effectiveness of our GPGL algorithm on MUTAG dataset with PointNet \cite{li2018pointcnn}, PointNet with 2D max-pooling, VGG16 \cite{Simonyan15VGG}, ResNet50 \cite{he2016deep}, and our MSM-CNN. In those networks, PointNet is a point-based network that has no 2D convolution or pooling except for a global max-pooling at last; by applying an image representation, we add a 2D max-pooling after each point-wise convolution layer to introduce 2D integration to image pixels. VGG16, ResNet50 are widely used 2D CNNs for image classification. Our MSM-CNN is also a 2D CNN for image classification. In all tests we use the same augmented data. From Table \ref{tab:ablation_CNN} we observe that 2D CNNs are significantly better than the PointNet with large margins of 21.66\%, and PointNet improves 19.68\% in accuracy by simply adding 2D max-pooling layers. The observation clearly shows that 2D CNNs benefit from our GPGL algorithm.

\begin{table}[h]
	\caption{\footnotesize Mean accuracy (\%) using different CNN classifiers. In Network parameters, K denotes thousand and M denotes million.}
	\label{tab:ablation_CNN}
	\vspace{-2mm}
	\adjustbox{width=\columnwidth}{
		\centering
		\begin{tabular}{c|ccccc}
			\toprule
			Network & PointNet & \begin{tabular}[c]{@{}c@{}}PointNet with\\2D pooling\end{tabular} & VGG16 & ResNet50 & MSM
			\\
			\midrule
			MUTAG (101$\times$) & 66.55 & 86.23 & 88.21 & 91.02 & 94.18 \\
			\midrule
			Network Parameters & 86.0K & 86.0K & 23.7M & 32.6M & 1.2M \\			
			\bottomrule
		\end{tabular}
	}
	\vspace{-1mm}
\end{table}

\begin{figure}[t]
	\begin{minipage}[b]{0.495\linewidth}
		\begin{center}
			\centerline{\includegraphics[width=\columnwidth]{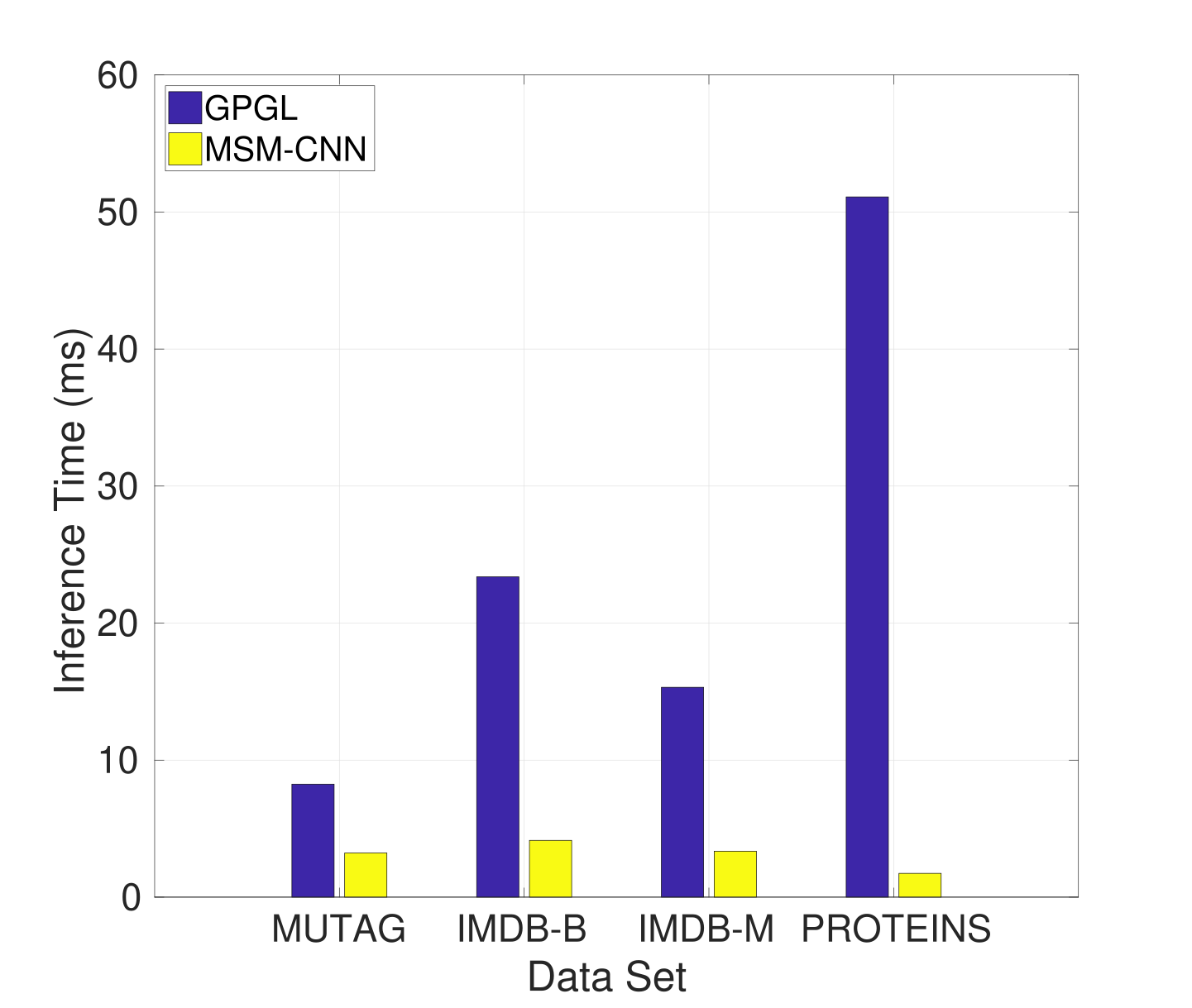}}
			\captionof{figure}{\footnotesize Running time at inference for GPGL and MSM-CNN.}
			\label{fig:running_time}
		\end{center}
		
	\end{minipage}
	\hfill
	\begin{minipage}[b]{0.495\linewidth}
		\begin{center}
			\centerline{\includegraphics[width=\columnwidth]{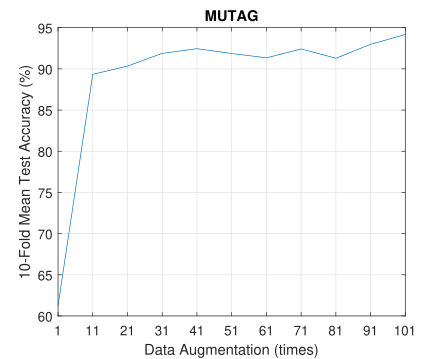}}
			\captionof{figure} {\footnotesize Illustration of data augmentation on classification.}
			\label{fig:augmentation}
		\end{center}
		
	\end{minipage}
	
	\vspace{-3mm}
\end{figure}

\bfsection{Effect of Data Augmentation using Grid Layouts on Classification}\label{subsec:MSM classification}
In order to train the deep classifiers well, the amount of training data is crucial. As shown in Alg.~\ref{alg:RKK_alg}, our method can easily generate tons of grid layouts that effectively capture different characteristics in the graph. Given the memory limit, we demonstrate the test performance for data augmentation in Fig.~\ref{fig:augmentation}, ranging from 1x to 101x with step 10x. As we see clearly, data augmentation can significantly boost the classification accuracy on MUTAG, and similar observations have been made for the other data sets.

\subsubsection{State-of-the-Art Comparison}
To do a fair comparison for graph classification, we follow the standard routine, \ie 10-fold cross-validation with random split. In the comparisons we compare our proposed method with existing works including graph convolution based methods and geometric deep learning based methods.
All the comparisons are summarized in Fig. \ref{fig:sota}. We generate 101x, 21x, 11x, 11x of the original size of MUTAG, IMDB-B, IMDB-M, and PROTEINS for data augmentation, and achieves $94.18\%\pm4.61\%, 74.90\%\pm4.01\%, 68.67\%\pm1.22\%, 79.52\%\pm1.72\%$ in terms of test accuracy, respectively. Among the tests, we achieve the state-of-the-art accuracy in IMDB-M and PROTEINS dataset. In MUTAG and IMDB-B datasets, Zhao \etal \cite{zhao2018work} achieves the state-of-the-art accuracy and ours is $0.82\%$ and $5.00\%$ behind.


\begin{figure*}[t]
	\begin{minipage}[b]{0.245\textwidth}
		\begin{center}
			\centerline{\includegraphics[width=1.05\columnwidth]{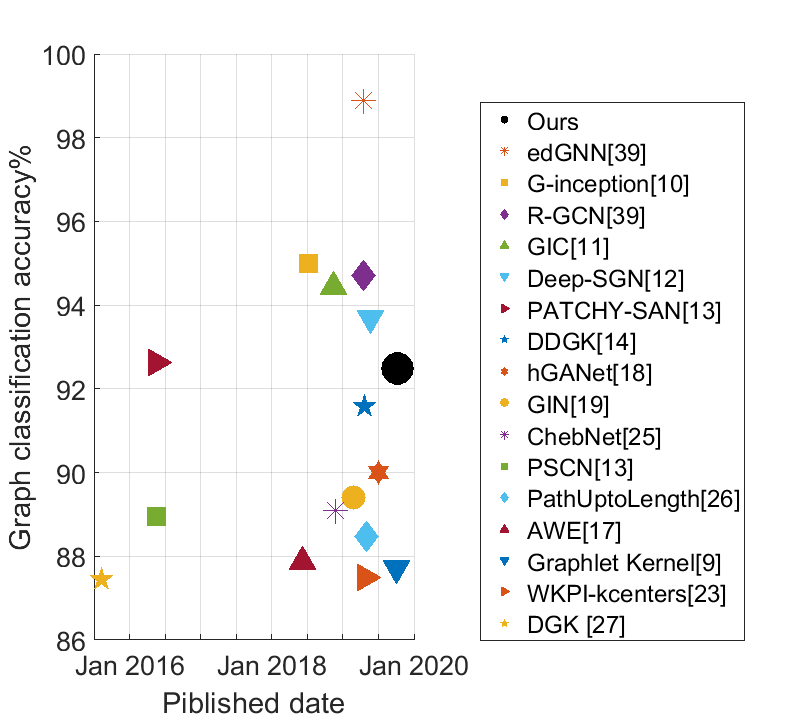}}
			\centerline{\footnotesize (a) MUTAG}
		\end{center}
	\end{minipage}
	\hfill
	\begin{minipage}[b]{0.245\textwidth}
		\begin{center}
			\centerline{\includegraphics[width=1.05\columnwidth]{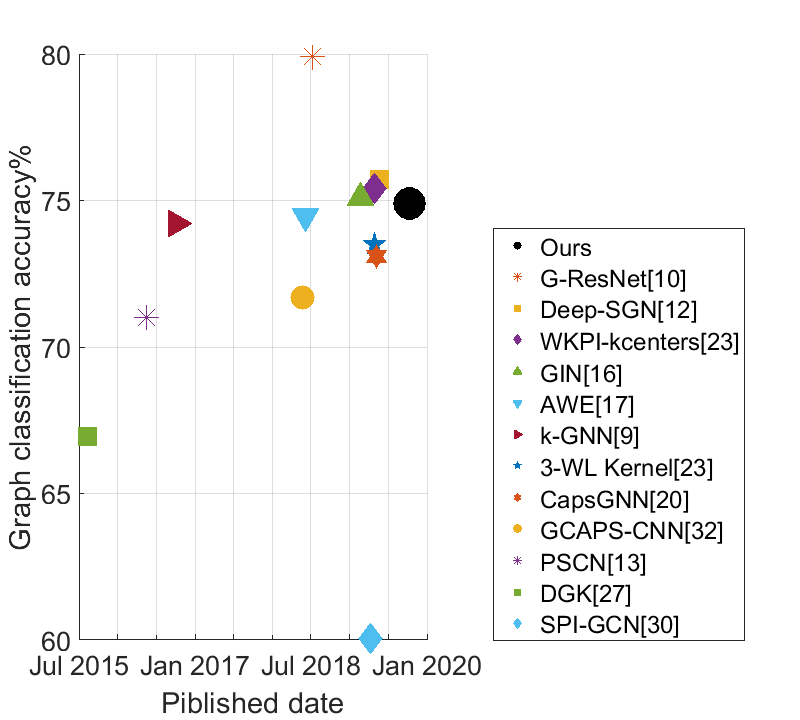}}
			\centerline{\footnotesize (b) IMDB-B}
		\end{center}
	\end{minipage}
	\hfill
	\begin{minipage}[b]{0.245\textwidth}
		\begin{center}
			\centerline{\includegraphics[width=1.05\columnwidth]{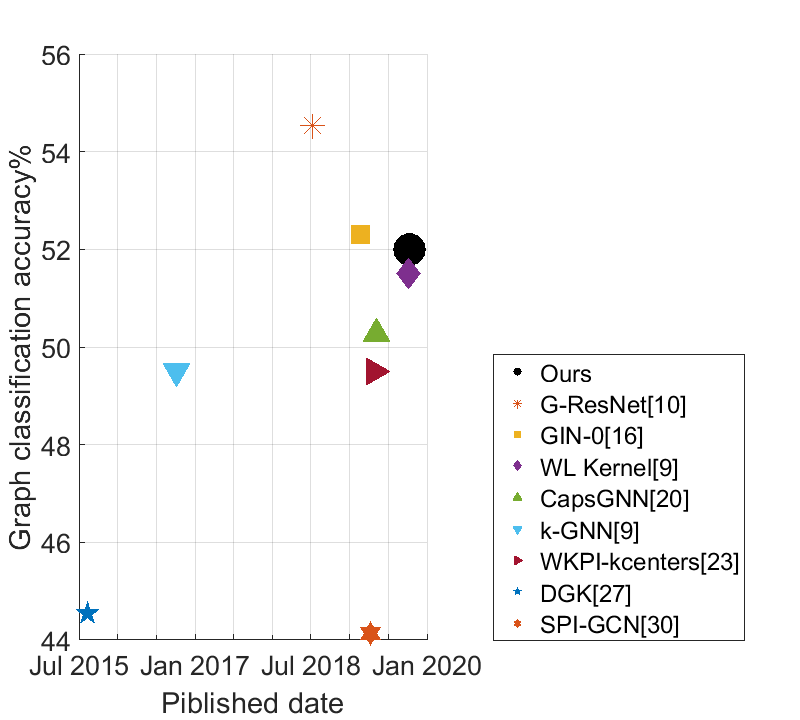}}
			\centerline{\footnotesize (c) IMDB-M}
		\end{center}
	\end{minipage}
	\hfill
	\begin{minipage}[b]{0.245\textwidth}
		\begin{center}
			\centerline{\includegraphics[width=1.05\columnwidth]{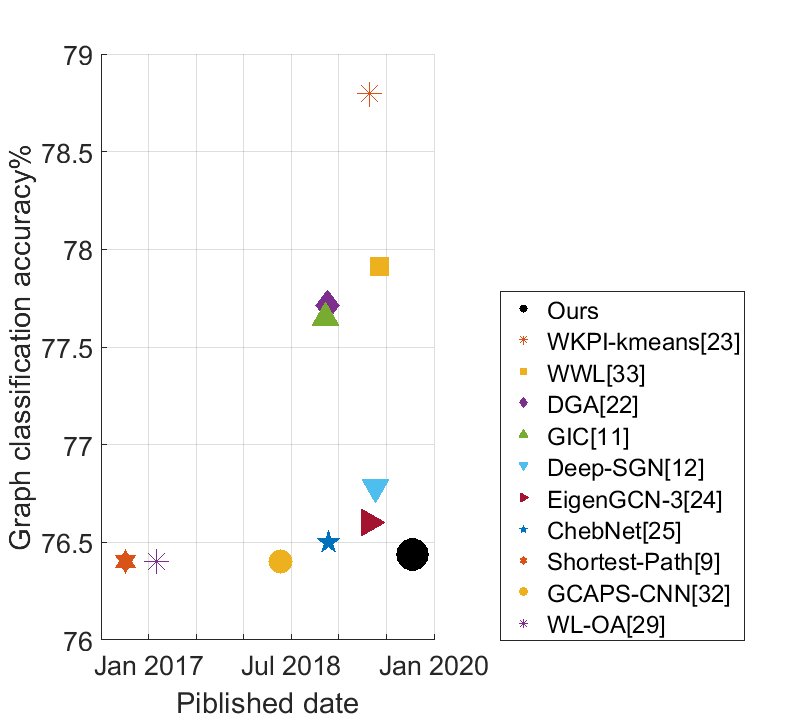}}
			\centerline{\footnotesize (d) PROTEINS}
		\end{center}
	\end{minipage}
	\hfill
	\caption{\footnotesize State-of-the-art result comparison. Numbers are cited from the leaderboard at \url{https://paperswithcode.com/task/graph-classification}}
	\label{fig:sota}
\end{figure*}

\begin{figure*}[t]
	\centering
	\includegraphics[width=.95\linewidth]{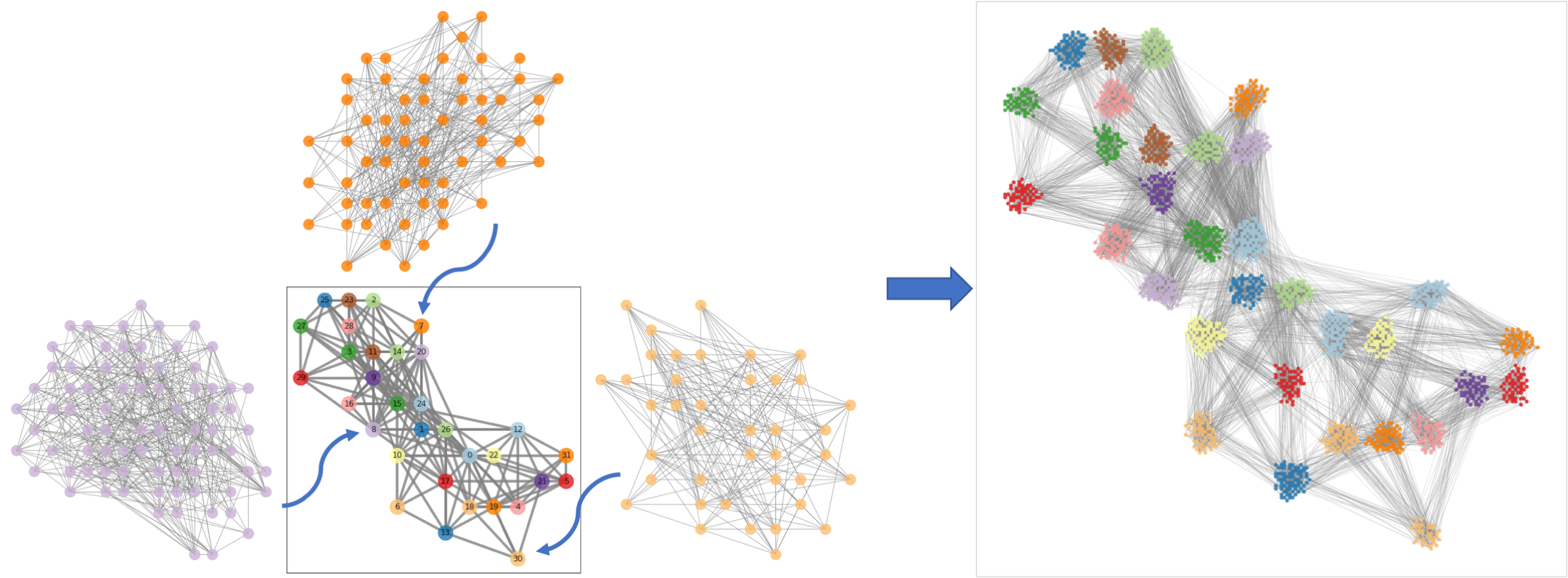}
	\caption{\footnotesize Illustration of two-level H-GPGL algorithm on a 2048-node graph, leading to a $256\times256$ grid layout.
	}
	\label{fig:H-GPGL_example}
\end{figure*}

\section{Hierarchical GPGL (H-GPGL)}
As we discussed before, the computational complexity of the KK algorithm is at least $O(|\mathcal{V}|^2)$ where $|\mathcal{V}|$ denotes the number of nodes in a graph. Also empirically from Fig. \ref{fig:running_time} we verify that even for small graphs the computational bottleneck of our method from GPGL (even this procedure can be offline) is significant that prevents ours from large-graph applications. For instance, the running time of GPGL on a simple graph (low node degree) with 2048 nodes takes about 90s. Therefore, in order to apply our method to large graphs we need to overcome this computational challenge. This motivates our hierarchical GPGL algorithm.


\begin{algorithm}[t]
	\SetAlgoLined
	\SetKwInOut{Input}{Input}\SetKwInOut{Output}{Output}
	\Input{number of nodes for GPGL $N$, undirected graph $\mathcal{G}=(\mathcal{V}, \mathcal{E})$, parent spatial location $\mathbf{a}_p$, parent grid size $\mathbf{s}_p$, child grid size $\mathbf{s}_c$, 2D grid layout $\mathcal{X}^*$}
	\Output{$\mathcal{X}^*$}
	\BlankLine
	
	$T\leftarrow round(\log_{N}{|\mathcal{V}|})$;
	
	\eIf{$T=1$}{
	$(\mathcal{X}, \mathcal{A}_c) \leftarrow \mbox{GPGL}(\mathcal{G}, \mathbf{a}_p, \mathbf{s}_c)$;
	
	$\mathcal{X}^*\leftarrow FitIntoGrid(\mathcal{X}^*, \mathcal{X}, \mathcal{A}_c)$;
	
	}{
	$\mathcal{H}\leftarrow partition(\mathcal{G}, N)$;
	
	Construct a connectivity graph $\mathcal{G}_H$ based on $\mathcal{H}$;
	
	$(\mathcal{X}_H, \mathcal{A}_H)\leftarrow \mbox{GPGL}(\mathcal{G}_H, \mathbf{a}_p, \mathbf{s}_p)$;
	
	\For{$i=1,\cdots,N$}{
	    $\mathcal{X}^*\leftarrow \mbox{H-GPGL}(N_i, \mathcal{H}_i, \mathcal{A}_{H_i}, \mathbf{s}_p, \mathbf{s}_c, \mathcal{X}^*)$;
	}
	}
	\Return $\mathcal{X}^*$;
	\caption{H-GPGL Algorithm}\label{alg:H-GPGL}
\end{algorithm}

\subsection{Algorithm Overview}

Our basic idea is to partition each graph into a set of subgraphs that can be organized in a hierarchy such as (complete) $m$-way tree where each node presents a subgraph and each level preserves the original graph connectivity among the subgraphs at the level. In this way we can preserve the original graph information as much as possible. To do so, an intuitive way is to partition a graph recursively, as listed in Alg. \ref{alg:H-GPGL}. We also illustrate an example of two-level GPGL generation procedure in Fig. \ref{fig:H-GPGL_example}. Here we first partition the graph with 2048 nodes into 32 subgraphs and map the connectivity graph of these subgraphs onto a $16\times16$ grid using the GPGL algorithm in Alg. \ref{alg:RKK_alg}. Then we project each subgraph again onto a $16\times16$ grid, leading to a $256\times256$ grid layout shown on the right side.  

\bfsection{Notations}
We define $N\in\mathbb{R}$ (\resp $N_i, \forall i$) as the number of nodes that GPGL can easily deal with. In the hierarchy, $\mathbf{a}_p\in\mathbb{R}^2$ denotes the 2D spatial location of a parent node in the grid $\mathcal{X}^*\subseteq\mathbb{R}^2$ and $\mathcal{A}_c\subseteq\mathbb{R}^2$ denotes the set of 2D spatial locations of the child nodes of the parent node. $\mathcal{H}$ denotes the set of partitioned subgraphs satisfying $\bigcup\mathcal{H}=\mathcal{G}$, and $\mathcal{X}_H\subseteq\mathbb{R}^2, \mathcal{A}_H\subseteq\mathbb{R}^2$ denote the grid layout and the spatial location set for subgraph connectivity graph $\mathcal{G}_H$ of $\mathcal{H}$. $\mathcal{H}_i\subseteq\mathbb{R}^2, \mathcal{A}_{H_i}\subseteq\mathbb{R}^2, \forall i$ denote the $i$-th subgraph in $\mathcal{H}$ and its spatial location, respectively. Note that there is one more output of $GPGL$ in Alg. \ref{alg:H-GPGL} than Alg. \ref{alg:RKK_alg}, which is the 2D location set for the input nodes. This operation can be easily implemented by adding an extra traversal on the grid. 

\bfsection{Fitting a Graph Node into the Grid Layout}
Let $\mathcal{P}_v$ denote the path from a leaf $v$ (\ie a graph node) to the root in the tree hierarchy through $M$ internal nodes. Then the 2D location of $v$, $\mathbf{a}_v$, on the grid layout $\mathcal{X}^*$ can be easily computed as follows:
\begin{align}
    \mathbf{a}_v = \tilde{\mathbf{a}}_{\mathcal{P}_v^{(0)}} +  \mathbf{s}_c\sum_{m=1}^M\tilde{\mathbf{a}}_{\mathcal{P}_v^{(m)}}\times\left(\mathbf{s}_p\right)^m,
\end{align}
where $\tilde{\mathbf{a}}_{\mathcal{P}_v^{(m)}}, \forall m$ denotes the relative location of the child node in the parent grid layout at the $m$-th level, and all the operators here are entry-wise. Note that in Alg. \ref{alg:H-GPGL} such localization calculation is done in a recursive way for the function $FitIntoGrid$. 

We illustrate this fitting procedure on the right side of Fig. \ref{fig:H-GPGL_example}.

\bfsection{Constructing Undirected Connectivity Graph of Subgraphs}
To build the connectivity graph $\mathcal{G}_H$ in Alg. \ref{alg:H-GPGL}, we first take each subgraph as one node in $\mathcal{G}_H$. Next we verify in the original graph whether there exists any edge that connects two subgraphs. If so, we add an edge between the two nodes in $\mathcal{G}_H$, otherwise, no edge. 


We illustrate this fitting procedure on the right side of Fig. \ref{fig:H-GPGL_example}. Firstly, we separated the 2048 points from the input point cloud into 32 clusters, and applied the GPGL algorithm to the cluster centers to create the high level layout. As shown in Fig. \ref{fig:H-GPGL_example}, the clusters labeled from 0 to 31 are projected to a $16\times16$ grid cell. For each cluster \eg cluster 8, we apply the GPGL algorithm to the cluster nodes and generate a low level image representation with size $16\times16$, which is shown as the purple graph in the figure. Finally, we embedded the low level image into the cell in the high level grid, \eg embedded the low level image of cluster 8 to the cell of cluster 8 in the high level grid. The grid cell without any cluster will be embedded with an empty $16\times16$ image. In that way we convert the 2048-node point cloud to a $256\times256$ image representation as illustrated on the right of Fig. \ref{fig:H-GPGL_example}.

\subsection{Normalized Cuts \cite{shi2000normalized} for Graph Partitioning}
Improving the computational efficiency of GPGL is our concern in developing H-GPGL, where the $partition$ function is the key. Recall that the computational complexity of GPGL is (at least) proportional to the square of the number of nodes in the graph. Now given an undirected graph $\mathcal{G}=(\mathcal{V}, \mathcal{E})$, the partition function should generate a set of $J$ subgraphs $\{\mathcal{G}_j=(\mathcal{V}_j, \mathcal{E}_j)\}_{j=1,\cdots,J}$ that minimize the following problem:
\begin{align}
    \min_{\{\mathcal{G}_j\}} \sum_{j=1}^J\||\mathcal{V}_i|\|^2, \, s.t. \, \bigcup\mathcal{V}_j = \mathcal{V}.
\end{align}
The solution of the optimization problem above suggests that ideally the sizes of subgraphs should be equal.

\setlength{\columnsep}{15pt}
\begin{wrapfigure}{r}{.45\linewidth}
	\vspace{-15pt}
	\begin{center}
		\includegraphics[width=\linewidth]{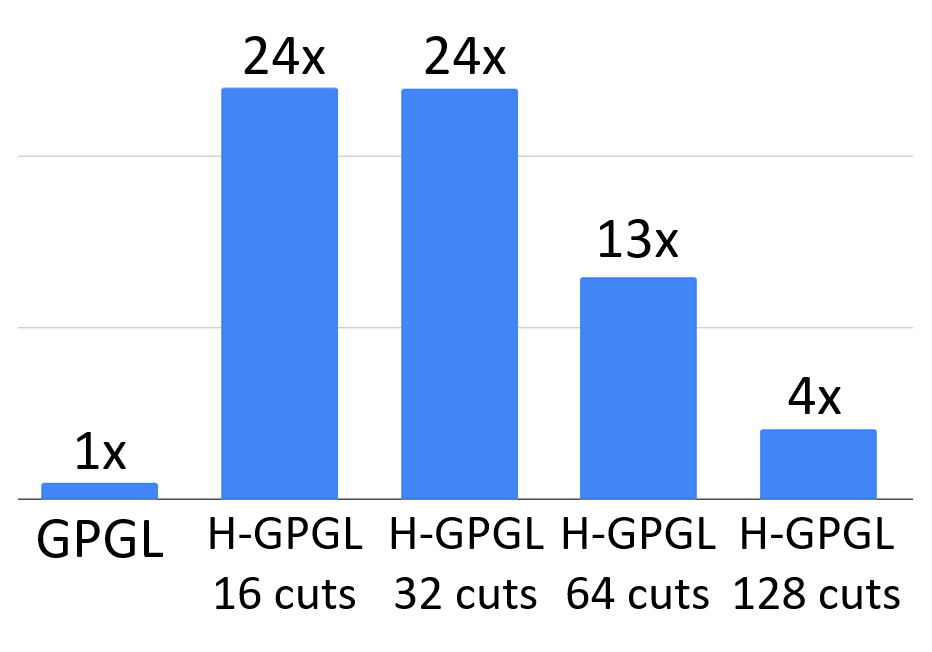}
		\vspace{-7mm}
		\caption{\footnotesize Running speed comparison on a 2048-node graph with two-level H-GPGL. See our experimental section later for more details.}
		\label{fig:H-GPGL-Speed}
	\end{center}
	\vspace{-15pt}
\end{wrapfigure}
This motivates us to employ the normalized graph cut algorithm as our partition function, because the algorithm aims to produce the subgraphs (given the number of subgraphs) with minimum cost as well as approximately equal sizes. Locating normalized cuts need to solve a generalized eigenvalue problem, whose computational complexity is essentially super-linear to the size of adjacency matrix of the graph \cite{vasudevan2017hierarchical}. Theoretically this complexity is no better than that of GPGL. Empirically, however, we do find that using normalized cut our H-GPGL performs much faster than GPGL alone. Fig. \ref{fig:H-GPGL-Speed} illustrates our running speed comparison result. With a proper number of cuts like 16 or 32 in our case, H-GPGL can process the graph within 5.5s, while GPGL needs about 90s. With more cuts, the running time of normalized cutting is longer and starts to dominate that of H-GPGL, leading to slower running speed.

\begin{table}[t]
	\caption{\footnotesize Running time (ms) comparison on a 2048 point cloud with 5-NN and Delaunay triangulation, followed by our H-GPGL.}
	\label{tab:Shapenet_running_speed}
	\vspace{-2mm}
	\adjustbox{width=\columnwidth}{
		\centering
		\begin{tabular}{c|ccc|c}
		    \toprule
          \begin{tabular}[c]{@{}c@{}c@{}}Method \end{tabular}
          & \begin{tabular}[c]{@{}c@{}c@{}}Graph\\construction \end{tabular}
          & \begin{tabular}[c]{@{}c@{}c@{}}Normalized\\cut \end{tabular}
          & \begin{tabular}[c]{@{}c@{}}GPGL \end{tabular}
          & \begin{tabular}[c]{@{}c@{}}Total \end{tabular} \\
          \midrule 
          5-NN & 74.02 & 2036.37 & 3621.70 & 5732.09 \\
          Delaunay & 1918.77 & 1703.07 & 1792.52 & 5414.36\\
			\bottomrule
		\end{tabular}
	}
\end{table}

\subsection{Experiment: Point Cloud Semantic Segmentation}
Consider a point cloud $\mathcal{P}\subseteq\mathbb{R}^3$ that contains a set of 3D points scanned from a single object. Each point $\mathbf{p}\in\mathcal{P}$ contains three geometry attributes $x,y,z$ representing its location in the 3D space, and a part label $c\in\mathcal{C}$ that the point $\mathbf{p}$ belongs to. Fig. \ref{fig:ShapeNet}(a) shows a point cloud with 2048 points scanned from a skateboard. The problem of point cloud semantic segmentation is defined as predicting the part labels for all the 3D points in a point cloud.

In this section we will show how to apply our H-GPGL algorithm to the point cloud semantic segmentation problem. To do so, we first construct a graph for each point cloud, then map each graph onto a grid using H-GPGL, and finally utilize a variant of multi-scale maxout CNN to conduct the segmentation.

\bfsection{Data Set}
For demonstration we use the ShapeNet \cite{yi2016scalable} part segmentation benchmark, which contains 16881 object samples and each object sample has 2048 point scans associated to one of the 50 part categories. The benchmark is split into an 14007-object training set and a 2874-object testing set.

\subsubsection{Graph Construction from Point Cloud}
Graph construction from point cloud is the first step in graph-based approaches for point cloud applications. The graph in the literature is usually generated by connecting the $K$ nearest neighbors (K-NN). However, the graph generation by K-NN often comes with difficulties in selecting a suitable $K$. On one hand, when $K$ is too small, the points are intended to form small subgraphs (clusters) with no guarantee of connectivity among the subgraphs. This makes graph-convolution fail to pass through all graph nodes. On the other hand, when $K$ is large, points are densely connected, leading to high processing load using graph kernels and large noise in local feature extraction.

In contrast to the K-NN based graph construction, our work employs the Delaunay triangulation \cite{delaunay1934sphere}, a widely-used triangulation method in computational geometry, to create graphs based on the positions of points. The triangulation graph has three advantages: (1) The connection of all the nodes in the graph is guaranteed, which makes graph-based network feasible on all constructed graphs from the point clouds; (2) All the local nodes are directly connected, which helps the local feature extraction; (3) The total number of graph connections is relatively low comparing to the graphs built by K-NN with a large $K$. Delaunay triangulation also returns better graphs that leads to better segmentation results (mcIoU \& miIoU: 83.8\%, 85.7\%) than 5-NN (mcIoU \& miIoU: 82.5\%, 84.3\%) using our MSM-CNN.

The worst-case computational complexity of Delaunay triangulation is well-known to be $O(n^{\lceil\frac{d}{2}\rceil})$ \cite{amenta2007complexity}, which in the 3D space is $O(n^2)$. Table \ref{tab:Shapenet_running_speed} lists the running time of graph construction algorithms followed by each key component in H-GPGL. Although Delaunay triangulation indeed needs significantly longer time, the overall running time is essentially better than 5-NN, because of better generated graphs.

\subsubsection{Implementation Details}
\bfsection{H-GPGL}
The 2048 points in each object sample is first mapped to a graph using Delaunay triangulation. Then we conduct two-level H-GPGL to generate 32 subgraphs using normalized cut at the parent level whose connectivity graph is mapped to a $16\times16$ grid. Each subgraph is mapped to a $16\times16$ grid as well, leading to a $256\times256$ grid layout for each point cloud. Finally we generate 14007 training samples and 2874 test samples for further usage. Note that for this task we do not use data augmentation, due to the computational bottleneck.

\bfsection{CNN for Segmentation}
In the experiment, we implement the following neural networks: (1) PointNet, (2) PointNet with 2D max-pooling and upsampling, (3) U-Net with VGG16 backbone, (4) U-Net with ResNet50 backbone, and (5) U-Net with MSM backbone. The architecture of PointNet with 2D max-pooing and upsampling is presented in Fig \ref{fig:Segmentation_PointNet_Pooling_Architecture} The U-Net based networks are implemented using the segmentation-models package \cite{Yakubovskiy2019sm}. At the end of the network, a mask filter is applied to label the void pixels as ignored category, which do not participate in loss calculation and gradient backpropagation. 

We train the network for 30 epochs using Adam \cite{kingma2014adam} with learning rate 0.0001 and batch-size 1. The training takes 15 hours for each neural network model on an NVidia 2080Ti GPU machine. We use two common metrics to evaluate the accuracy, mean instance intersection over union (miIoU) and mean class intersection over union (mcIoU). 

\begin{figure}[t]
	\centering
	\includegraphics[width=.9\linewidth]{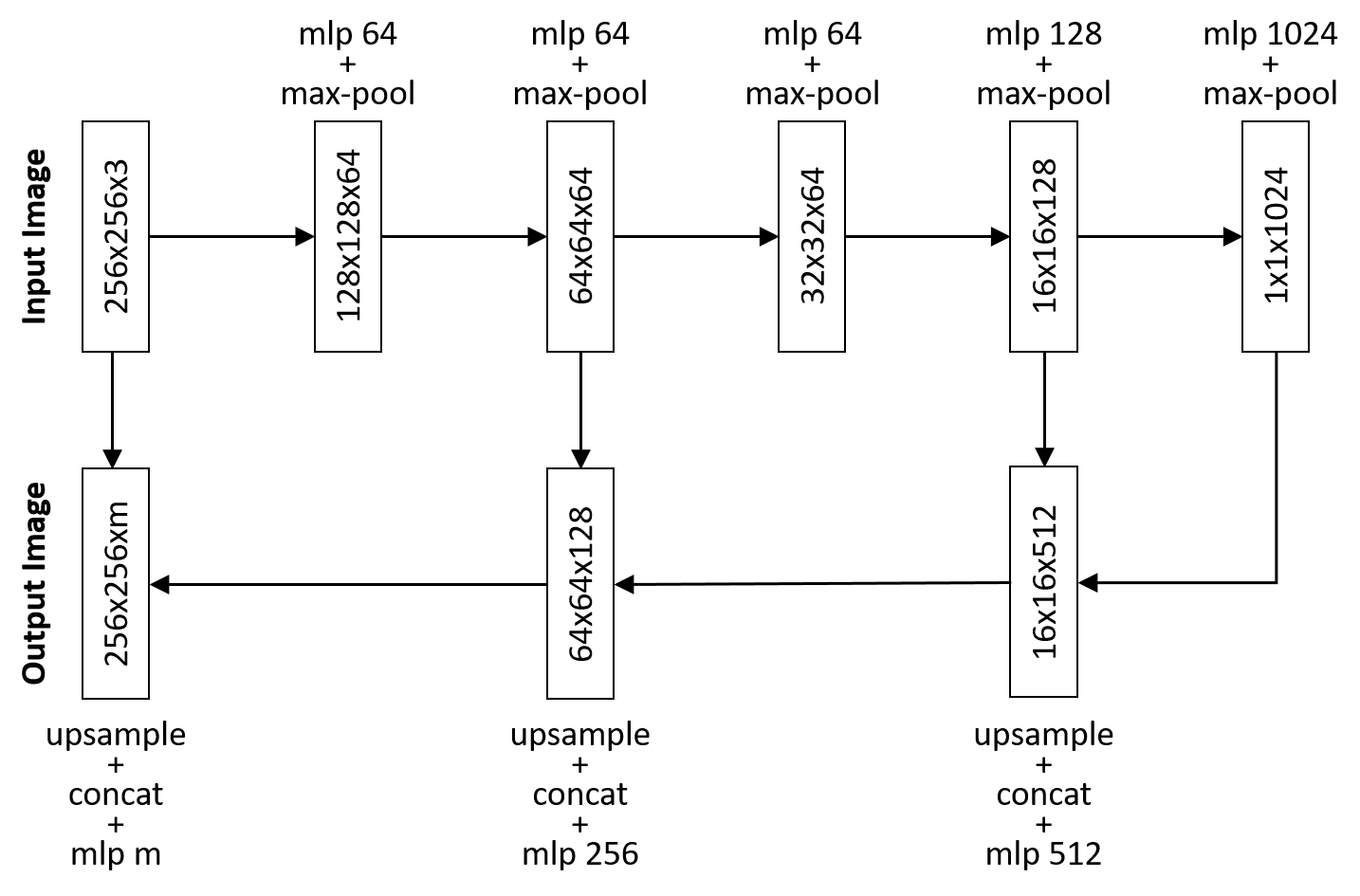}
	\caption{\footnotesize Architecture of PointNet with 2D max-pooling and upsampling for point cloud segmentation.}
	\label{fig:Segmentation_PointNet_Pooling_Architecture}
	\vspace{-3mm}
\end{figure}

\bfsection{Segmentation on overlapped points}
The H-GPGL on a point cloud may lead to part of the points overlapped in the same grid cell. Once overlapped points occur, we take an average, by default, of all the points and assign it to the grid cell. After network inference, the grid cell is labeled with a predicted semantic category, which is then assigned to all the overlapped points associated to that cell. In that way, we fuse the features of overlapped points and distribute the grid cell labels to those points.


\begin{figure}[t]
	\begin{minipage}[b]{0.15\textwidth}
		\begin{center}
			\centerline{\includegraphics[width=\columnwidth]{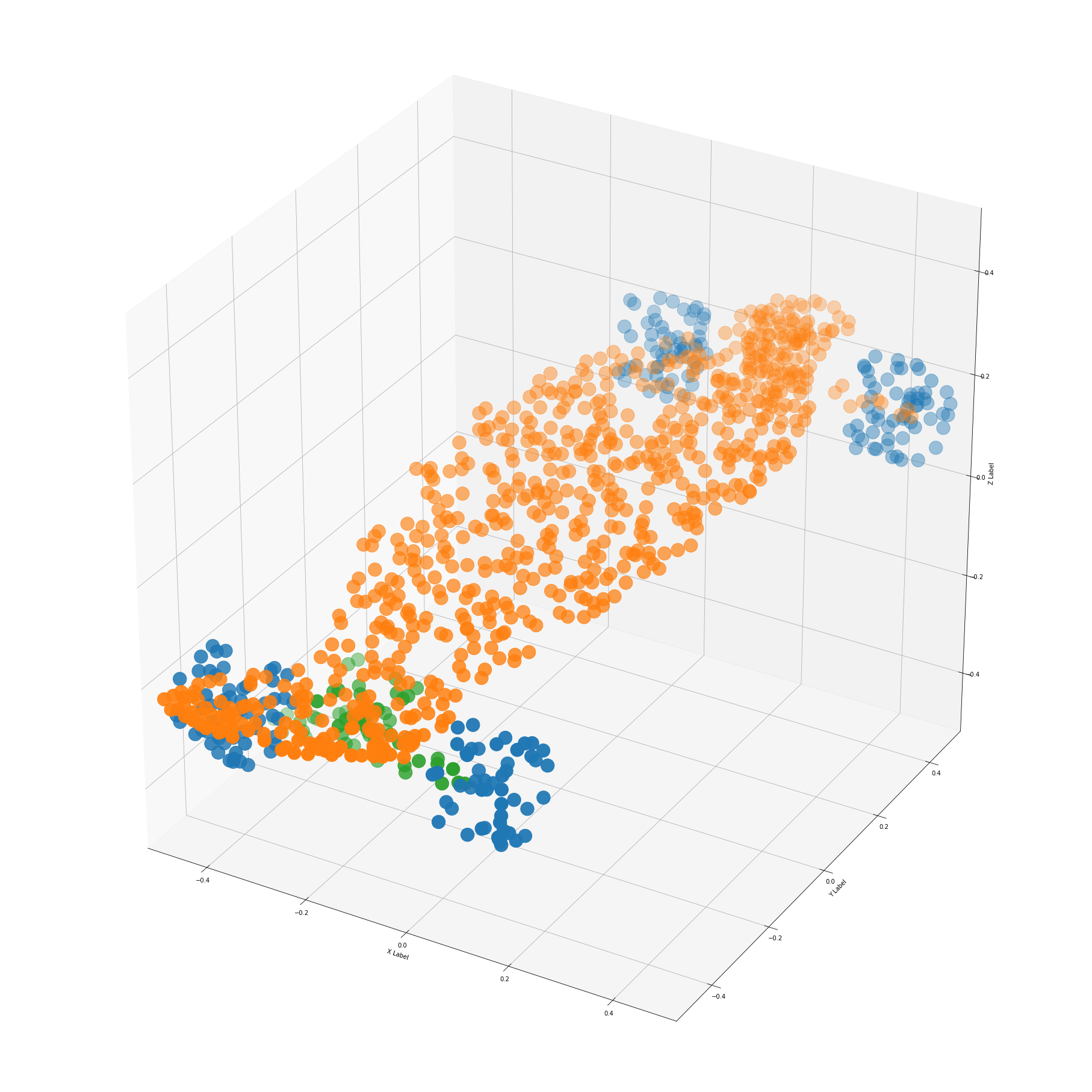}}
			\centerline{\footnotesize (a) Ground Truth}
		\end{center}
	\end{minipage}
	\hfill
	\begin{minipage}[b]{0.15\textwidth}
		\begin{center}
			\centerline{\includegraphics[width=\columnwidth]{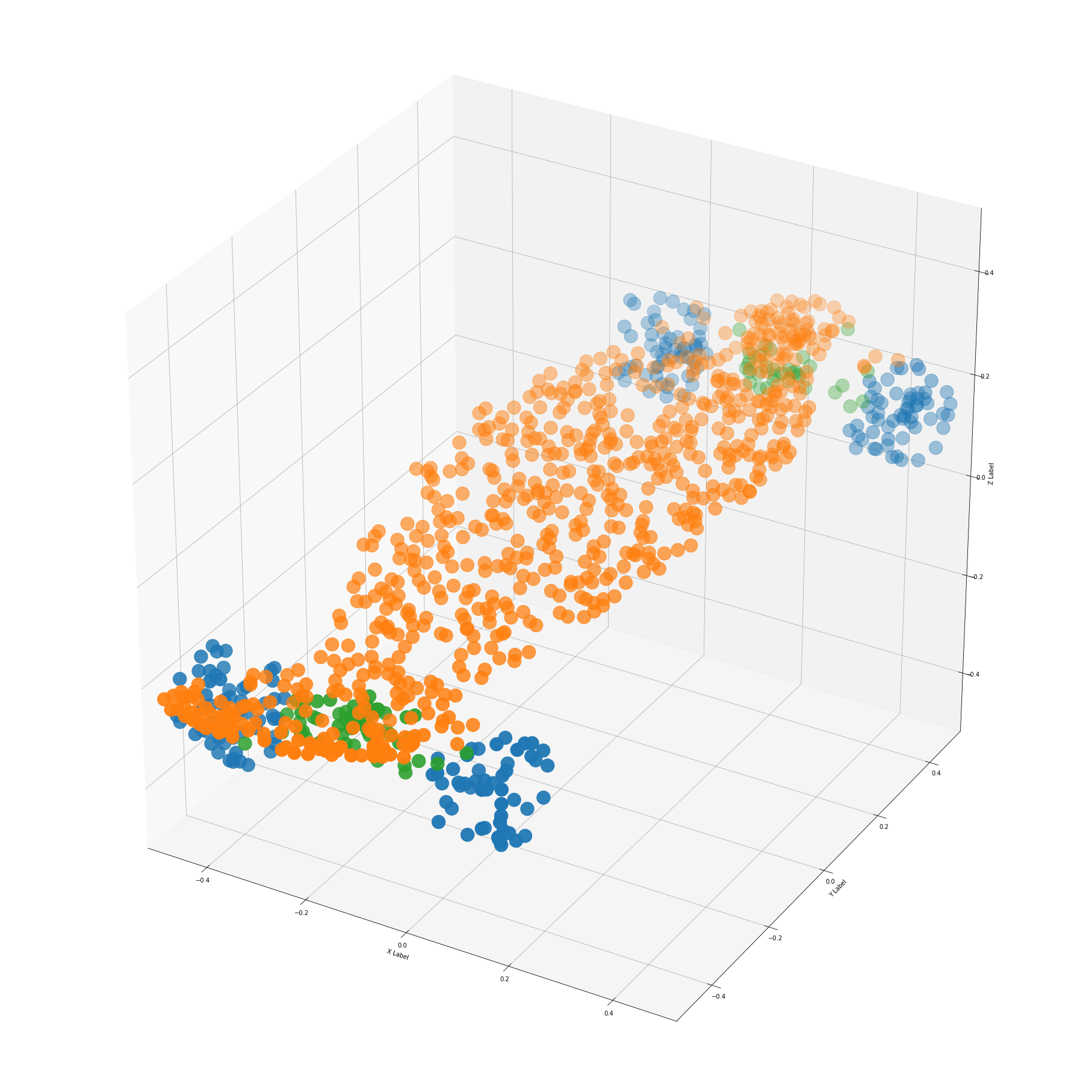}}
			\centerline{\footnotesize (b) Ours}
		\end{center}
	\end{minipage}
	\hfill
	\begin{minipage}[b]{0.15\textwidth}
		\begin{center}
			\centerline{\includegraphics[width=\columnwidth]{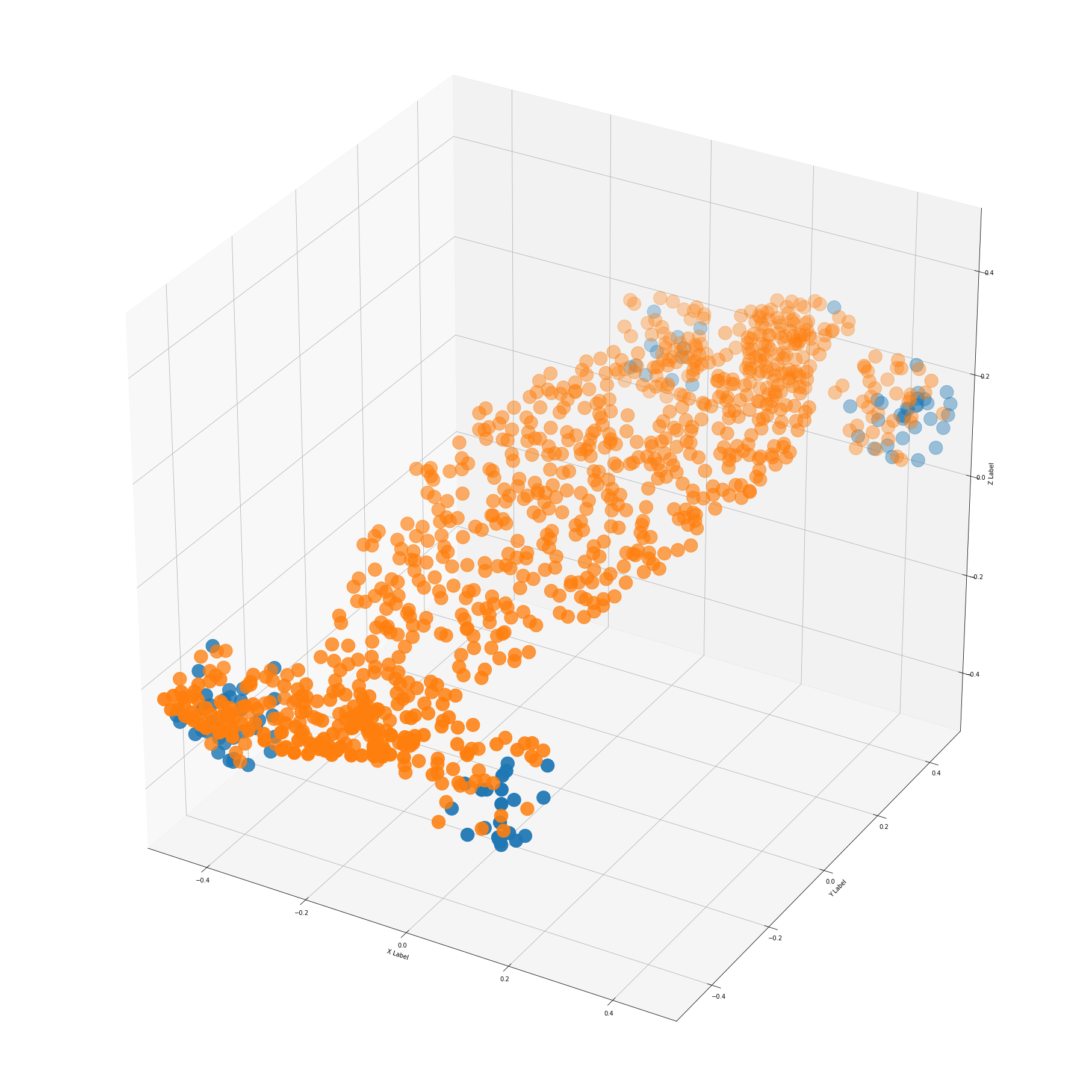}}
			\centerline{\footnotesize (c) PointNet}
		\end{center}
	\end{minipage}
	\hfill
	\vspace{-3mm}
	\caption{\footnotesize Visualization comparison of skateboard part segmentation results from {\bf(a)} ground truth, {\bf (b)} our H-GPGL, and {\bf (c)} PointNet. }
	\label{fig:ShapeNet}
	\vspace{-3mm}
\end{figure}


\begin{figure}[t]
	\centering
	\includegraphics[width=.95\linewidth]{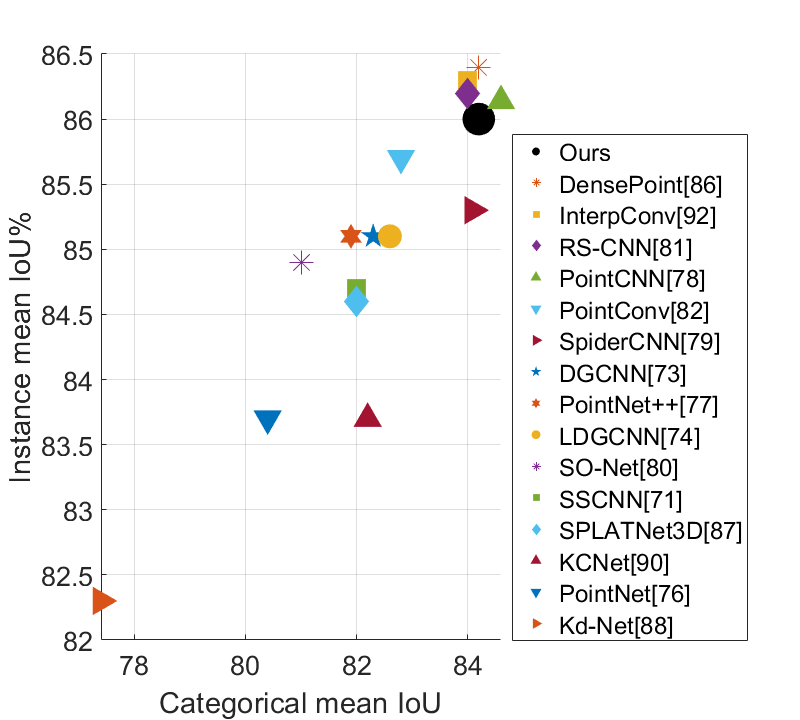}
	\caption{\footnotesize State-of-the-art segmentation result comparison on ShapeNet.}
	\label{fig:Result_Comparison_ShapNet}
	\vspace{-3mm}
\end{figure}

\subsubsection{Results}
Table \ref{tab:Result_Seg} presents the mcIoU and miIoU of all the neural network models on ShapeNet point cloud segmentation task. From the table we observe that the PointNet model achieves similar performance as in \cite{qi2017pointnet}, while PointNet with 2D max-pooling and upsampling gains an improved performance with a large margin of 1.9\% on mcIoU and 1.9\% on miIoU. We also observe that all U-Net based models achieve an improved performance against the PointNet model, which indicates that 2D CNNs benefit from the GPGL image representations in point cloud feature learning.

\begin{table}[h]
	\caption{\footnotesize Results of ShapeNet point cloud segmentation using different CNN models. The mcIoU denotes mean class intersection over union, and the miIoU denotes mean instance intersection over union. }
	\label{tab:Result_Seg}
	\vspace{-2mm}
	\adjustbox{width=\columnwidth}{
		\centering
		\begin{tabular}{c|ccccc}
			\toprule
			Network & PointNet & \begin{tabular}[c]{@{}c@{}}PointNet with\\2D pooling\end{tabular} & VGG16 & ResNet50 & MSM
			\\
			\midrule
			mcIoU & 80.1 & 82.0 & 82.6 & 83.1 & 84.2 \\
			\midrule
			miIoU & 80.4 & 82.3 & 82.8 & 84.0 & 86.0 \\
			\bottomrule
		\end{tabular}
	}
	\vspace{-1mm}
\end{table}

\begin{table*}[t]
	\caption{\footnotesize Result comparison on ShapeNet part segmentation IoU (\%) of each object class.}
	\label{tab:Shapenet_Result_per_class}
	\vspace{-2mm}
	\adjustbox{width=\textwidth}{
		\centering
		\begin{tabular}{c|cccccccccccccccc|c}
		    \toprule
            & air plane & bag & cap & car
            & chair & ear phone & guitar & knife
            & lamp  & laptop & motor bike & mug
            & pistol & rocket & skateboard & table & ave. \\
            \midrule 
            PointNet \cite{qi2017pointnet}
            & 83.4 & 78.7 & 82.5 & 74.9
            & 89.6 & 73.0 & 91.5 & 85.9
            & 80.8 & 95.3 & 65.2 & 93.0
            & 81.2 & 57.9 & 72.8 & 80.6 & 80.4\\
            
            Pointnet++ \cite{qi2017pointnet++}
            & 82.4 & 79.0 & 87.7 & 77.3
            & 90.8 & 71.8 & 91.0 & 85.9
            & 83.7 & 95.3 & 71.6 & 94.1
            & 81.3 & 58.7 & 76.4 & 82.6 & 81.9\\
            
            DGCNN \cite{wang2019dynamic}
            & \bf{84.2} & 83.7 & 84.4 & 77.1
            & 90.9 & 78.5 & 91.5 & 87.3
            & 82.9 & \bf{96.0} & 67.8 & 93.3
            & 82.6 & 59.7 & 75.5 & 82.0 & 82.3\\
            
            RS-CNN \cite{liu2019relation}
            & 83.5 & \bf{84.8} & 88.8 & \bf{79.6}
            & \bf{91.2} & 81.1 & 91.6 & 88.4
            & 86.0 & \bf{96.0} & 73.7 & 94.1
            & 83.4 & 60.5 & 77.7 & 83.6 & 84.0\\
            
            DensePoint \cite{liu2019densepoint}
            & 84.0 & 85.4 & \bf{90.0} & 79.2
            & 91.1 & \bf{81.6} & 91.5 & 87.5
            & 84.7 & 95.9 & \bf{74.3} & \bf{94.6}
            & 82.9 & 64.6 & 76.8 & 83.7 & \bf{84.2}\\
            \midrule
            {\bf Ours}
            & 83.3 & 83.1 & 89.4 & 75.9
            & 87.8 & 80.5 & \bf{91.7} & \bf{90.9}
            & \bf{87.2} & \bf{96.0} & 66.6 & 92.6
            & \bf{86.3} & \bf{69.4} & \bf{81.1} & \bf{86.0} & \bf{84.2}\\
			\bottomrule
		\end{tabular}
	}
	\vspace{-3mm}
\end{table*}

We visualize our segmentation result in Fig. \ref{fig:ShapeNet} comparing with the ground truth as well as the one using PointNet. As we see, our result is much closer to the ground truth, especially on the top-left wheel and the axis between the two bottom-wheels. Although we have some wrong predictions among the points between the two top-wheels, we think that these mistakes are acceptable because the shape is symmetric that makes the prediction harder to differentiate the top and bottom parts.

We then summarize the state-of-the-art segmentation results on ShapeNet in Fig. \ref{fig:Result_Comparison_ShapNet}. Clearly our results using MSM-CNN, both mcIoU and miIoU, are comparable to the literature. Further we list the ShapeNet part segmentation IoU of each object class in Table \ref{tab:Shapenet_Result_per_class} for more detail comparison, and ours are the best among 8 of the 16 classes.

\bfsection{Impact of overlapped points}
To understand the impact of overlapped points in point cloud segmentation, we analyze the segmentation accuracy of all points and overlapped points on the ShapeNet test set. Among the 2874 test samples with 5,885,952 points, 17616 points are overlapped to the same grid cell, which covers 0.299\% of the points. In the test, 62.7\% of the overlapped points are mislabeled, which leads to a 0.22\% decrease in mcIoU and 0.21\% in miIoU. This observation shows that the MSM-CNN have difficulty in segmenting the overlapped points, which needs further study to overcome. However, since the overlapped points take only a tiny portion of the point cloud, they have a minor impact on the accuracy of point cloud segmentation.

\section{Conclusion}
In this paper we answer the question positively that CNNs can be used directly for graph applications by projecting graphs on grids properly. To this end, we propose a novel graph drawing problem, namely graph-preserving grid layout (GPGL), which is an integer programming to learn 2D grid layouts by minimizing topology loss. We propose a penalized Kamada-Kawai algorithm to solve the integer programming and a multi-scale maxout CNN to work with GPGL. We manage to demonstrate the success of GPGL on small graph classification. To improve the computational efficiency of GPGL, we propose hierarchical GPGL (H-GPGL) that utilizes graph partitioning algorithms such as normalized cut to generate subgraphs which GPGL is applied to. We demonstrate that H-GPGL is much more suitable than GPGL to large graph applications such as 3D point cloud semantic segmentation, where we achieve the state-of-the-art on ShapeNet part segmentation benchmark. As future work we are interested in applying this method to real-world problems such as LiDAR data processing.

\section*{Acknowledgment}
This work was supported in part by the Mitsubishi Electric Research Laboratories (MERL) and NSF CCF-2006738. 


\ifCLASSOPTIONcaptionsoff
  \newpage
\fi


\bibliographystyle{IEEEtran}
\bibliography{egbib}

\vskip -5pt plus -1fil

\begin{IEEEbiography}[{\includegraphics[width=1in,height=1.25in,clip,keepaspectratio]{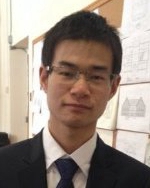}}]{Yecheng Lyu}
(S'17) received his B.S. degree from Wuhan University, China in 2012 and M.S. degree from Worcester Polytechnic Institute, USA in 2015 where he is currently a Ph.D student working on autonomous vehicles. His current research interest is point cloud processing and deep learning.
\end{IEEEbiography}

\vskip -5pt plus -1fil

\begin{IEEEbiography}[{\includegraphics[width=1in,height=1.25in,clip,keepaspectratio]{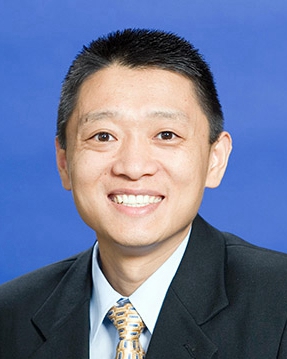}}]{Xinming Huang}
(M’01–SM’09) received the Ph.D. degree in electrical engineering from Virginia Tech in 2001. Since 2006, he has been a faculty in the Department of Electrical and Computer Engineering at Worcester Polytechnic Institute (WPI), where he is currently a chair professor. Previously he was a Member of Technical Staffs with the Bell Labs of Lucent Technologies. His main research interests are in the areas of circuits and systems, with emphasis on autonomous vehicles, deep learning, IoT and wireless communications. 
\end{IEEEbiography}

\vskip -5pt plus -1fil

\begin{IEEEbiography}[{\includegraphics[width=1in,height=1.25in,clip,keepaspectratio]{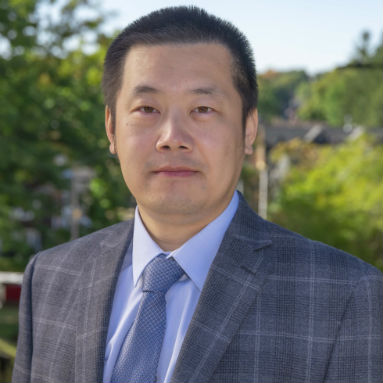}}]{Ziming Zhang}
is an assistant professor at Worcester Polytechnic Institute (WPI). Before joining WPI he was a research scientist at Mitsubishi Electric Research Laboratories (MERL) in 2017-2019. Prior to that, he was a research assistant professor at Boston University in 2016-2017. Dr. Zhang received his PhD in 2013 from Oxford Brookes University, UK, under the supervision of Prof. Philip H. S. Torr. His research areas include object recognition and detection, zero-shot learning, deep learning, optimization, large-scale information retrieval, visual surveillance, and medical imaging analysis. His works have appeared in TPAMI, IJCV, CVPR, ICCV, ECCV, ACM Multimedia, ICDM, ICLR and NIPS. He won the R\&D 100 Award 2018.
\end{IEEEbiography}

\end{document}